\documentclass{article}

\usepackage[utf8]{inputenc} 
\usepackage[T1]{fontenc}    
\usepackage{hyperref}       
\usepackage{url}            
\usepackage{booktabs}       
\usepackage{amsfonts}       
\usepackage{nicefrac}       
\usepackage{microtype}      
\usepackage{bm}
\usepackage{amsmath}
\usepackage{amssymb}
\usepackage{xcolor}
\usepackage{mathtools}
\usepackage{mathrsfs}
\usepackage{tikz}
\usepackage{xspace}
\usepackage{algorithm}
\usepackage{amsthm}
\usepackage{sidecap}
\usepackage{adjustbox}
\usepackage{dsfont}
\usepackage[skins]{tcolorbox}

\usetikzlibrary{patterns}
\sidecaptionvpos{figure}{c}

\newtcbox{\blueinlinebox}[1][]{
 enhanced,
 before=\adjustbox{valign=c}\bgroup,
 after=\egroup,
 colback=blue!20,
 colframe=white,
 coltext=black,
 size=small,
 left=0pt,
 right=0pt,
 boxsep=2pt,
 #1
}
\newtcbox{\orangeinlinebox}[1][]{
 enhanced,
 before=\adjustbox{valign=c}\bgroup,
 after=\egroup,
 colback=orange!50,
 colframe=white,
 coltext=black,
 size=small,
 left=0pt,
 right=0pt,
 boxsep=2pt,
 #1
}
\newtcbox{\redinlinebox}[1][]{
 enhanced,
 before=\adjustbox{valign=c}\bgroup,
 after=\egroup,
 colback=red!30,
 colframe=white,
 coltext=black,
 size=small,
 left=0pt,
 right=0pt,
 boxsep=2pt,
 #1
}
\newtcbox{\greeninlinebox}[1][]{
 enhanced,
 before=\adjustbox{valign=c}\bgroup,
 after=\egroup,
 colback=green!30,
 colframe=white,
 coltext=black,
 size=small,
 left=0pt,
 right=0pt,
 boxsep=2pt,
 #1
}

\newcommand{\pfa}{\ensuremath{\mathcal{P}}\xspace} 
\newcommand{\apfa}{\ensuremath{\widetilde{\mathcal{P}}}\xspace} 
\newcommand{\refp}{\mathbf{r}}					
\newcommand{\vq}{\mathbf{y}}					
\newcommand{\vl}{\mathbf{l}}					
\newcommand{\vu}{\mathbf{u}}					

\newcommand{\poi}{\ensuremath{\operatorname{PoI}}\xspace}					
\newcommand{\hv}{\ensuremath{\operatorname{HV}}\xspace}
\newcommand{\hvi}{\ensuremath{\Delta}\xspace}
\newcommand{\dom}{\ensuremath{\mbox{dom}}\xspace} 
\newcommand{\ndom}{\ensuremath{\mbox{ndom}}\xspace} 

\newcommand{\X}{\ensuremath{\mathrm{X}}\xspace} 
\newcommand{\Y}{\ensuremath{\mathrm{Y}}\xspace} 
\newcommand{\E}{\ensuremath{\mathcal{D}}\xspace} 

\newcommand{\diag}{\ensuremath{\operatorname{diag}}\xspace}
\newcommand{\domain}{\ensuremath{\mathcal{X}}\xspace} 
\newcommand{\R}{\ensuremath{\mathbb{R}}\xspace}

\newcommand{\BF}[1]{
	\relax
	\ifmmode
	\ifcat\noexpand#1\relax 
		\boldsymbol{#1}     
	\else
		\mathbf{#1}
	\fi
	\else
		\textbf{#1}
	\fi
}

\newcommand{\PDF}{\ensuremath{\operatorname{PDF}}}
\newcommand{\CDF}{\ensuremath{\operatorname{CDF}}}
\newcommand{\GP}{\ensuremath{\operatorname{gp}}}
\newcommand{\ud}{\ensuremath{\mathrm{d}}}
\newcommand{\floor}[1]{\ensuremath{\lfloor#1\rfloor}}

\newtheorem{theorem}{Theorem}
\newtheorem{lemma}{Lemma}
\newtheorem*{remark}{Remark}

\newtcbox{\blueinlineboxDots}[1][]{
 enhanced,
 before=\adjustbox{valign=c}\bgroup,
 after=\egroup,
 colback=blue!20,
 colframe=white,
 coltext=black,
 size=small,
 left=0pt,
 right=0pt,
 boxsep=2pt,
 overlay={%
    \fill[pattern=dots] (frame.south west) rectangle (frame.north east);
 }
 #1
}
\newtcbox{\orangeinlineboxDots}[1][]{
 enhanced,
 before=\adjustbox{valign=c}\bgroup,
 after=\egroup,
 colback=orange!50,
 colframe=white,
 coltext=black,
 size=small,
 left=0pt,
 right=0pt,
 boxsep=2pt,
 overlay={%
     \fill[pattern=dots] (frame.south west) rectangle (frame.north east);
    }
 #1
}
\newtcbox{\redinlineboxDots}[1][]{
 enhanced,
 before=\adjustbox{valign=c}\bgroup,
 after=\egroup,
 colback=red!30,
 colframe=white,
 coltext=black,
 size=small,
 left=0pt,
 right=0pt,
 boxsep=2pt,
 overlay={%
    \fill[pattern=dots] (frame.south west) rectangle (frame.north east);
 }
 #1
}
\newtcbox{\greeninlineboxDots}[1][]{
 enhanced,
 before=\adjustbox{valign=c}\bgroup,
 after=\egroup,
 colback=green!30,
 colframe=white,
 coltext=black,
 size=small,
 left=0pt,
 right=0pt,
 boxsep=2pt,
 overlay={%
    \fill[pattern=dots] (frame.south west) rectangle (frame.north east);
 }
 #1
}

\usepackage{graphicx}
\usepackage{caption}
\usepackage{subcaption}

\definecolor{indiagreen}{rgb}{0.07, 0.53, 0.03}
\newcommand{\reddot}{\raisebox{0.6pt}{\tikz{\fill[red!100] (0,0) circle (0.08)}}}
\newcommand{\greendot}{\raisebox{0.6pt}{\tikz{\fill[indiagreen!100] (0,0) circle (0.08)}}}
\newcommand{\orangedot}{\raisebox{0.6pt}{\tikz{\fill[orange!100] (0,0) circle (0.08)}}}

\usepackage{microtype}
\usepackage{graphicx}
\usepackage{booktabs} 

\usepackage{hyperref}

\usepackage[accepted]{icml2024}

\usepackage{amsmath}
\usepackage{amssymb}
\usepackage{mathtools}
\usepackage{amsthm}

\usepackage[capitalize,noabbrev]{cleveref}

\theoremstyle{plain}

\theoremstyle{definition}

\theoremstyle{remark}

\usepackage[textsize=tiny]{todonotes}

\icmltitlerunning{Probability Distribution of Hypervolume Improvement in Bi-objective Bayesian Optimization}

\begin{document}

\twocolumn[
\icmltitle{Probability Distribution of Hypervolume Improvement in Bi-objective Bayesian Optimization}



\icmlsetsymbol{equal}{*}

\begin{icmlauthorlist}
\icmlauthor{Hao Wang}{leiden}
\icmlauthor{Kaifeng Yang}{fhooe}
\icmlauthor{Michael Affenzeller}{fhooe}
\end{icmlauthorlist}

\icmlaffiliation{leiden}{Leiden University, Leiden, The Netherlands}
\icmlaffiliation{fhooe}{University of Applied Sciences, Hagenberg, Austria}

\icmlcorrespondingauthor{Kaifeng Yang}{kaifeng.yang@fh-hagenberg.at}

\icmlkeywords{Machine Learning, ICML}

\vskip 0.3in
]



\printAffiliationsAndNotice{}  

\begin{abstract}
Hypervolume improvement (HVI) is commonly employed in multi-objective Bayesian optimization algorithms to define acquisition functions due to its Pareto-compliant property.
Rather than focusing on specific statistical moments of HVI, this work aims to provide the exact expression of HVI's probability distribution for bi-objective problems. Considering a bi-variate Gaussian random variable resulting from Gaussian process (GP) modeling, we derive the probability distribution of its hypervolume improvement via a cell partition-based method. Our exact expression is superior in numerical accuracy and computation efficiency compared to the Monte Carlo approximation of HVI's distribution.
Utilizing this distribution, we propose a novel acquisition function - $\varepsilon$-probability of hypervolume improvement ($\varepsilon$-PoHVI). Experimentally, we show that on many widely-applied bi-objective test problems, $\varepsilon$-PoHVI significantly outperforms other related acquisition functions, e.g., $\varepsilon$-PoI, and expected hypervolume improvement, when the GP model exhibits a large the prediction uncertainty.
\end{abstract}

\section{Introduction}\label{sec:intro}
For solving black-box multi-objective optimization problems (MOPs), the hypervolume indicator (HV)~\cite{zitzler1999multiobjective} is extensively employed for assessing the quality of the Pareto front approximation or guiding the search direction. HV is defined as the Lebesgue measure of the subset of $\mathbb{R}^m$ dominated by an approximation set to the Pareto front. It is extensively applied in many multi-objective optimization algorithms, e.g., indicator-based evolutionary algorithms~\cite{BeumeNE07,DebAPM02} and Bayesian optimization~\cite{EmmerichYD0F16,YangEDB19,DaultonBB20,Emmerich2020,DaultonBB20,DaultonBB21,SuzukiTTSK20,GarridoMerchanFH23,ZuluagaKP16,YangDYBE16}. In multi-objective Bayesian optimization, HV induces the famous hypervolume improvement (HVI) function~\cite{emmerich2005single}, which quantifies the benefit of appending a new data point to the approximation set - the increment of the HV value caused by the new point. HVI generalizes the notion of ``improvement'' in the single-objective scenarios, and therefore, it serves as the base of many successful multi-objective acquisition functions, e.g., the probability of improvement (PoI)~\cite{emmerich2006single,keane2006statistical} that measures the chance of realizing nonzero HVI, $\varepsilon$-PoI that computes the probability of objective points having least $\varepsilon$ distance to the approximation set,  and the expected hypervolume improvement (EHVI)~\cite{emmerich2006single,YangEDB19} that generalizes expected improvement (EI)~\cite{JonesSW98} from single-objective Bayesian optimization.

\paragraph{Motivation} At an unknown decision point, the predictive/posterior distribution of its objective value follows a multivariate normal distribution in Bayesian optimization. Consequently, HVI defined on this predictive distribution is a real-valued random variable. To the best of our knowledge, the existing multi-objective acquisition functions either only consider the first moment of HVI's distribution, e.g., EHVI for the mean, or completely disregard the distribution, e.g., $\varepsilon$-PoI, which is not related to the quantile of HVI's distribution (see Sec.~\ref{sec:acquisition}). However, when HVI shows a large dispersion (e.g., large variance), only relying on the mean of HVI makes the acquisition function less trustworthy/meaningful. In this sense, higher moments or at least the quantiles can help quantify the risk/uncertainty of the acquisition value~\cite{MehlawatGK21,schonlau1998global}, which is difficult to obtain without the exact expression of HVI's distribution.  

\paragraph{Contributions} To answer the issue we raise above, we aim to provide the exact distribution function of HVI and demonstrate its usefulness in Bayesian optimization for speeding up empirical convergence. Our contributions are summarized as follows: 
\begin{enumerate}
    \item We derive the exact expression of HVI's distribution function in the bi-objective optimization scenario and numerically validate it against the Monte Carlo (MC) method. The exact distribution exhibits better computational efficiency and numerical accuracy (Sec.~\ref{sec:hvi-distribution})
    \item We propose a novel acquisition function, $\varepsilon$-Probability of Hypervolume Improvement ($\varepsilon$-PoHVI), which utilizes HVI's distribution function directly. It computes the probability of making at least $\varepsilon$ hypervolume improvement to the current approximation set (Sec.~\ref{sec:acquisition}).
    \item We compare $\varepsilon$-PoHVI, $\varepsilon$-PoI, and EHVI on 14 selected test problems, where we observe $\varepsilon$-PoHVI substantially improves the empirical convergence of Bayesian optimization over $\varepsilon$-PoI and EHVI (Sec.~\ref{sec:experiments}).
\end{enumerate}   

\section{Preliminaries}\label{sec:background}
\paragraph{Multi-objective optimization}
A real-valued multi-objective optimization problem (MOP) involves minimizing multiple objective functions simultaneously, i.e., $\BF{f}=(f_1, \ldots, f_m), f_i:\domain\rightarrow \R$, $\domain \subseteq \R^d, i\in[1..m]$. For every $\mathbf{y}^{(1)}$ and $\mathbf{y}^{(2)}\in \mathbb{R}^m$, we say $\mathbf{y}^{(1)}$ weakly dominates $\mathbf{y}^{(2)}$ (written as $\mathbf{y}^{(1)} \preceq \mathbf{y}^{(2)}$) iff. $y^{(1)}_i \leq y^{(2)}_i$, $i\in[1..m]$. The Pareto order $\prec$ on $\mathbb{R}^m$ is defined: $\mathbf{y}^{(1)} \prec \mathbf{y}^{(2)}$ iff. $\mathbf{y}^{(1)} \preceq \mathbf{y}^{(2)}$ and $\mathbf{y}^{(1)} \neq \mathbf{y}^{(2)}$. A point $\mathbf{x} \in \mathcal{X}$ is efficient iff. $\nexists \BF{x}'\in\domain(\BF{f}(\BF{x}') \prec \BF{f}(\BF{x}))$. The set of all efficient points of $\mathcal{X}$ is called the \emph{efficient set}.
The image of the efficient set under $\mathbf{f}$ is called the \emph{Pareto front}. 
Multi-objective optimization algorithms often employ a finite multiset $\mathrm{X} = \{\mathbf{x}^{(1)}, \dots,\mathbf{x}^{(n)}\}$ to approximate the efficient set, whose image under $\BF{f}$ is denoted by $\Y$. The \emph{non-dominated subset} of $\Y$ is a finite approximation to the Pareto front, which is denoted by $\mathcal{P}$.
\emph{Non-dominated space} w.r.t.~$\pfa$ is the subset of $\R^m$ that is not dominated by \pfa, i.e., $\mbox{ndom} (\pfa):= \{ \BF{y} \in \R^m \colon \nexists \mathbf{p} \in \pfa(\BF{p} \prec \BF{y})\}$. Similarly, the \emph{dominated space} w.r.t. $\pfa$, denoted by $\mbox{dom}(\pfa)$, is the complement of $\mbox{ndom}(\pfa)$. 

\paragraph{Bayesian Optimization} BO~\cite{Mockus74,JonesSW98,ShahriariSWAF16} is a sequential model-based optimization algorithm for solving black-box optimization problems that are expensive to evaluate. BO starts with sampling a small initial set of data points $\X \subseteq \domain$ (with Latin Hypercube Sampling). After evaluating $\X$ with $\BF{f}$, it constructs a probabilistic model $\Pr(\BF{f} \mid \X, \Y)$ (e.g., Gaussian process regression). BO quantifies the quality of unseen points with an acquisition function, which targets balancing exploration and exploitation of the search process.
BO chooses the next point to evaluate by maximizing the acquisition function. Please see~\cite{Knowles06,EmmerichYD0F16,Emmerich2020,DaultonBB20,BelakariaDD19,ZhangG20,TuGKS22,GarridoMerchanFH23} for more details and developments on this topic.

\paragraph{Gaussian process regression}
In this work, we model each objective function independently as the realization of a centered
Gaussian process (GP) prior~\cite{RasmussenW06}, i.e., $f_i \sim \GP(0, k_i), i\in[1..m]$, where $k_i\colon \domain \times \domain \rightarrow \mathbb{R}$ is a kernel function
that models the auto-covariance of $f_i$, $\forall \BF{x},\BF{x}'\in\mathcal{X},  \operatorname{Cov}\{f_i(\BF{x}),f_i(\BF{x}^{\prime})\}=k_i(\BF{x}, \BF{x}^{\prime})$. Given a data set $\E = (\X, \Y)$, $\X=\{\BF{x}^{(1)},\ldots,\BF{x}^{(n)}\}\subset \domain$, and $\Y = \{\mathbf{f}(\mathbf{x}^{(1)}), \dots, \mathbf{f}(\mathbf{x}^{(n)})\}$,
the posterior GPs are independent: $f_i \mid \E \sim \GP(\hat{f}_i, \hat{k}_{i})$, where $\hat{f}_i$ and $\hat{k}_i$ are the posterior means and kernel functions, respectively.
Many works have been devoted to model cross-correlations among GPs~\cite{AlvarezRL12}, e.g., multi-task GP~\cite{BonillaCW07} and dependent GP~\cite{BoyleF04}. 

\paragraph{Hypervolume Improvement}
The hypervolume (HV) indicator of a set $\pfa\subseteq \R^m$ is defined as the Lebesgue measure $\lambda$ of the set that is dominated by $\pfa$ and bounded from above by a reference point $\refp\in\R^m$, i.e., $\hv(\pfa,\refp)=\lambda(\{\BF{y}\in\R^m \colon \BF{y} \prec \refp \wedge \exists\BF{p}\in\pfa, \BF{p}\prec\BF{y}\})$.
The hypervolume indicator is often taken as a performance metric for comparing the empirical performance for multi-objective optimization algorithms~\cite{ZitzlerTLFF03} or used in the indicator-based optimization algorithms~\cite{BeumeNE07}. In Bayesian optimization, the set $\pfa$ is typically the approximation set - the non-dominated subset of the data $\mathcal{D}$. 
The contribution of a single objective vector $\mathbf{y}$ to $\pfa$ can be quantified by the well-known \emph{hypervolume improvement} (HVI):
\begin{equation} \label{eq:hvi}
    \Delta^{+}
    (\BF{y}; \pfa, \refp)=\hv(\pfa\cup\{\BF{y}\}, \refp) - \hv(\pfa, \refp).
\end{equation}
Note that, for a dominated point, i.e., $\BF{y}\in\operatorname{dom}(\pfa)$, its HVI value is zero. Here, we introduce a ``plus'' symbol in $\Delta^{+}$ to indicate that HVI is non-negative. The purpose of this notation shall become clear when we propose a generalization to the definition of HVI (see Sec.~\ref{sec:hvi-distribution}).
HVI underpins many useful acquisition functions in Bayesian optimization. For instance, \emph{Probability of Improvement} (PoI)~\cite{stuckman1988global} is originally devised for single-objective optimization cases and later generalized to multi-objective optimization~\cite{emmerich2006single,keane2006statistical}. It quantifies the probability that $\BF{y}$ lies in the non-dominated space w.r.t.~\pfa: $\poi(\BF{x}; \mathcal{P})=\operatorname{E}\{\mathds{1}_{\mbox{ndom}(\mathcal{P})}(\BF{y})\mid \E, \BF{x}\}$.
Also, $\varepsilon$-\emph{Probability of Improvement} ($\varepsilon$-PoI) is proposed recently~\cite{Emmerich2020} to make the search less exploitative. It measures the probability of the non-dominated points that are at least $\varepsilon$ away from the approximation set \pfa: 
$$
\operatorname{\varepsilon-\poi}(\BF{x}; \mathcal{P}, \varepsilon) = \operatorname{E}\left\{\mathds{1}_{\mbox{ndom}(\mathcal{P})}(\BF{y}+\varepsilon\BF{1}_m)\mid \E, \BF{x}\right\},
$$
where $\BF{1}_m$ is an $m$-dimensional vector of 1's. The computational complexity of $\operatorname{\varepsilon-\poi}$ is $\Theta(n \log n)$ for $m=2,3$~\cite{yang2017computing,Emmerich2020} and $\mathcal{O}(2^{m-1} n^{\floor{\frac{m}{2}}})$ for $m \geq 4$~\cite{yang2019efficient}, where $n=|\pfa|$. Furthermore, \emph{Expected hypervolume improvement}~\cite{emmerich2006single} calculates the expectation of the HVI value of a multivariate Gaussian random variable: $\mbox{EHVI}(\BF{x}; \pfa, \refp)\coloneqq \operatorname{E}\{\Delta^+(\BF{y}; \pfa, \refp)\mid \mathcal{D}, \BF{x} \}$. The time complexity of EHVI's computation is $\Theta(n \log n)$ for $m=2,3$~\cite{EmmerichYD0F16,yang2017computing,yang2019efficient} and $\mathcal{O}(2^{m-1} n^{\floor{\frac{m}{2}}})$ for $m \geq 4$~\cite{yang2019efficient}. 

\section{The Distribution of Hypervolume Improvement} \label{sec:hvi-distribution}
\paragraph{Generalized hypervolume improvement} Prior to deriving the distribution functions, we first propose to generalize the definition of HVI in Eq.~\eqref{eq:hvi} for assigning nonzero values to the dominated points. We define the negative hypervolume ``improvement'' of a dominated point $\BF{y}$ as the \emph{negative} volume of the intersection of the set that dominates $\BF{y}$ with the set dominated by \pfa, namely,
$$
\Delta^{-}(\BF{y}) = -\lambda\left(\{\BF{p}\in\R^m\colon \BF{p}\preceq \BF{y}\}\cap\dom(\pfa)\right),
$$
which penalizes the dominated points that are located far from \pfa. In Fig.~\ref{fig:2dhv}, we depict an example for the negative HVI (point $\BF{b}$). We claim that the negative HVI is desirable/useful to compute for two major reasons: (1) 
It makes the acquisition optimization step more tractable when a large subset of objective points are dominated by \pfa w.h.p., typically when \pfa is quite close to the Pareto front. Since the HVI is nearly zero on this subset, the acquisition function defined on HVI (e.g., PoI and EHVI) will exhibit a large plateau on its optimization landscape, making it harder to maximize. The negative HVI practically mitigates this issue by turning plateaus into valleys. (2) HVI is also extensively employed in evolutionary multi-objective algorithms (EMOAs)~\cite{DebAPM02,BeumeNE07}, where it is necessary to assign nonzero values to dominated points to move such points to the Pareto front. 
Our extension, the negative HVI, facilitates a sensible comparison among dominated points, which is
a more natural extension compared to other existing proposals~\cite{WangRDE15}. 
We combine the negative HVI with Eq.~\eqref{eq:hvi}, resulting in the \emph{generalized hypervolume improvement}:
\begin{equation} \label{eq:generalized-hvi}
    \Delta(\BF{y}) =\mathds{1}_{\ndom(\pfa)}(\BF{y})\Delta^{+}(\BF{y})  + \mathds{1}_{\dom(\pfa)}(\BF{y})\Delta^{-}(\BF{y}).
\end{equation}
Importantly, if one wishes to focus only on the non-dominated points, then the above generalization is exactly the same with Eq.~\eqref{eq:hvi}. 

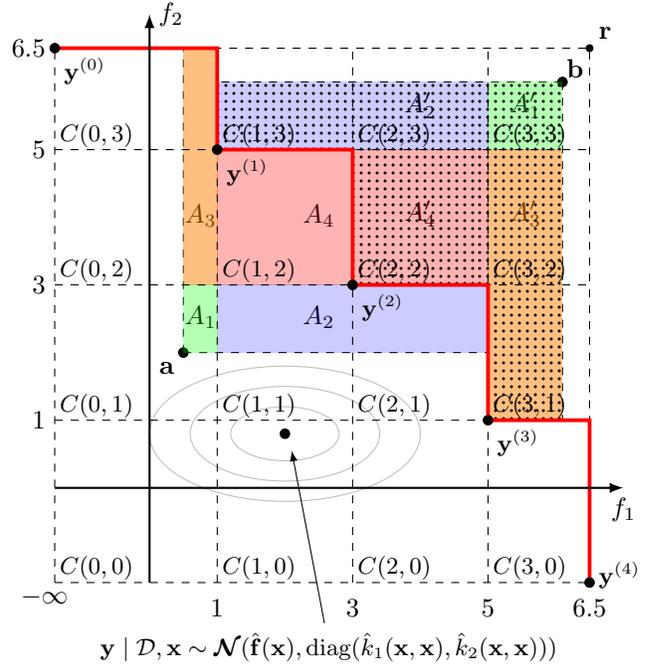
\begin{figure}[!ht]
\begin{tikzpicture}[scale=.9]
    \def\inf{1.4}
    \def\r{0.07}
    \def\refx{6.5}
    \def\refy{6.5}
    \def\x{{-\inf,1,3,5,\refx}}
    \def\y{{-\inf,1,3,5,\refy}}
    \def\pfax{{-\inf,1,3,5,\refx}}
    \def\pfay{{\refy,5,3,1,-\inf}}    
	\node [below left,xshift=3mm,yshift=.5mm,fill opacity=1] at (-\inf, -\inf) {$-\infty$};
    \draw [color=black!25] (2,0.8) ellipse [x radius=2, y radius=1, xscale=1, yscale=1];
    \draw [color=black!25] (2,0.8) ellipse [x radius=2, y radius=1, xscale=.7, yscale=.7];
    \draw [color=black!25] (2,0.8) ellipse [x radius=2, y radius=1, xscale=.4, yscale=.4];
    \draw [color=black,fill=black,fill opacity=1] (2,0.8) circle (\r);
    \node [anchor=north] (distribution) at (2.65, -2) {\small$\BF{y} \mid \E, \BF{x} \sim \bm{\mathcal{N}}(\hat{\BF{f}}(\BF{x}), \diag(\hat{k}_1(\BF{x}, \BF{x}), \hat{k}_2(\BF{x}, \BF{x})))$};
    \draw [-latex] (distribution) -- (2.1, 0.55);
    \draw [color=black, fill=black,fill opacity=1] (0.5, 2) circle (0.07);
    \node [below left,yshift=0mm,fill opacity=1] at (0.5, 2)   {$\BF{a}$};
    \node [above,xshift=2.3mm,yshift=2mm,fill opacity=1] at (0.5, 2)   {$A_1$};
    \node [above,xshift=0mm,yshift=2mm,fill opacity=1] at (2.5, 2)   {$A_2$};
    \node [above,xshift=2.3mm,yshift=2mm,fill opacity=1] at (0.5, 3.5)   {$A_3$};
    \node [above,xshift=0mm,yshift=2mm,fill opacity=1] at (2.5, 3.5)   {$A_4$};
    \node [above,xshift=2.3mm,yshift=2mm,fill opacity=1] at (5.3, 5.1)   {$A_1'$};
    \node [above,xshift=0mm,yshift=2mm,fill opacity=1] at (4, 5.1)   {$A_2'$};
    \node [above,xshift=2.3mm,yshift=2mm,fill opacity=1] at (5.3, 3.5)   {$A_3'$};
    \node [above,xshift=0mm,yshift=2mm,fill opacity=1] at (4, 3.5)   {$A_4'$};
    \draw [dashed] (0.5,2) -- (0.5,\refy);
    \draw [dashed] (0.5,2) -- (5,2);
    \path [fill=orange, opacity=0.5] (0.5,3) rectangle (1,\refy);
    \path [fill=red, opacity=0.3] (1,3) rectangle (3,5);
    \path [fill=blue, opacity=0.2] (1,2) rectangle (5,3);
    \path [fill=green, opacity=0.3] (0.5,2) rectangle (1,3);
    \draw [color=black, fill=black,fill opacity=1] (6.1, 6) circle (\r);
    \node [above right,fill opacity=1,xshift=-.7mm,yshift=-.7mm] at (6.1, 6) {$\BF{b}$};
    \draw [dashed] (6.1,6) -- (6.1,1);
    \draw [dashed] (6.1,6) -- (1,6);
    \path [pattern=dots,pattern color=black][preaction={fill=orange, opacity=0.5}] (5,1) rectangle (6.1,5);
    \path [pattern=dots,pattern color=black][preaction={fill=red, opacity=0.3}] (3,3) rectangle (5,5);
    \path [pattern=dots,pattern color=black][preaction={fill=blue, opacity=0.2}] (1,5) rectangle (5,6);
    \path [pattern=dots,pattern color=black][preaction={fill=green, opacity=0.3}] (5,5) rectangle (6.1,6);
	\foreach \i in {1,3,5,\refy}{
         \node [left,xshift=0mm,fill opacity=1] at (-\inf,\i) {$\i$};
    }
    \foreach \i in {1,3,5,\refy}{
         \node [below,yshift=-1mm,fill opacity=1] at (\i, -\inf)   {$\i$};
    }
    \foreach \i in {0,...,4}{
         \draw [dashed] (\x[\i],-\inf) -- (\x[\i],\refy);
         \draw [dashed] (-\inf,\y[\i]) -- (\refx,\y[\i]);
    }
    \foreach \i in {0,...,3}{
        \draw [red,line width=0.5mm] (\pfax[\i], \pfay[\i]) -- (\pfax[\i+1], \pfay[\i]) -- (\pfax[\i+1], \pfay[\i+1]);
    }
    \foreach \i in {0,...,3}{
    	\draw [color=black,fill=black,fill opacity=1] (\pfax[\i], \pfay[\i]) circle (\r);
	    \node [below right, fill opacity=1] at (\pfax[\i], \pfay[\i]) {\small $\mathbf{y}^{(\i)}$};
    }
    \draw [color=black,fill=black,fill opacity=1] (\pfax[4], \pfay[4]) circle (\r);
	\node [right,yshift=1mm,fill opacity=1] at (\pfax[4], \pfay[4]) {\small $\mathbf{y}^{(4)}$};
	\draw [color=black, fill=black,fill opacity=1] (\refx, \refy) circle (0.05);
	\node [above right,fill opacity=1] at (\refx, \refy) {$\mathbf {r}$};
	\draw [-latex,thick] (-\inf,0) -- (\refx + .5,0);
	\node [below,fill opacity=1] at (\refx+.5,0) {$f_1$};
	\draw [-latex,thick] (0,-\inf) -- (0,\refy + .5);
	\node [right,fill opacity=1] at (0,\refy + .5) {$f_2$};
    \footnotesize
    \foreach \i in {0,...,3}{
        \foreach \j in {0,...,3}{
            \node [above right] at (\x[\i]-0.05,\y[\j]-0.08) {$C(\i,\j)$};
        }
    }
\end{tikzpicture}
\caption{
\label{fig:2dhv}
For a two-dimensional objective space, we picture the augmented Pareto approximation set \apfa by the black dots $\BF{y}^{(0)},\ldots,\BF{y}^{(4)}$ and the attainment boundary by the red curve. The posterior distribution of $\BF{y}$ at a point $\mathbf{x}\in\mathcal{X}$ is illustrated by the light gray ellipsoids. The generalized hypervolume improvement of two realizations $\BF{a}$ and $\BF{b}$ are depicted in the shaded area.
The objective space $[-\BF{\infty}, \BF{r}]$ is partitioned into cells (e.g., $C(1,0)$). When restricting the random point $\BF{y}$ to a cell, its hypervolume can always be expressed in four terms: $\hvi^+(\BF{a})|_{C(0,1)}=\lambda($\greeninlinebox{$A_1$}$)+\lambda($\blueinlinebox{$A_2$}$)+\lambda($\orangeinlinebox{$A_3$}$)+ \lambda($\redinlinebox{$A_4$}) ($\lambda$ is the Lebesgue measure on $\R^2$). 
When a point is dominated by \apfa, its negative hypervolume improvement can be written similarly: $\hvi^-(\BF{b})|_{C(3,3)}=-\lambda($\greeninlineboxDots{$A_1'$}$)-\lambda($\blueinlineboxDots{$A_2'$}$)-\lambda($\orangeinlineboxDots{$A_3'$}$)-\lambda($\redinlineboxDots{$A_4'$}).
}
\end{figure}
Assume a data set $\mathcal{D}=(\mathrm{X}, \mathrm{Y})$ observed on the vector-valued objective function. The approximation set $\pfa\subset \mathbb{R}^m$ to the Pareto front is the non-dominated subset of $\mathrm{Y}$. Also, we assume a reference point $\refp\in \R^m$.
For the bi-objective case ($m=2$), we depict an example of the HVI in Fig.~\ref{fig:2dhv}. In this example, the random point $\mathbf{y}$ follows the posterior distribution of a Gaussian process model, i.e., $\BF{y} \mid \E, \BF{x} \sim \bm{\mathcal{N}}(\hat{\BF{f}}(\BF{x}), \diag(\hat{k}_1(\BF{x}, \BF{x}), \hat{k}_2(\BF{x}, \BF{x})))$.
It can be observed from the figure that the expression of $\hvi^{+}(\BF{y})$ depends on the subset of $\pfa$ that it dominates, indicating the expression of $\hvi^{+}(\BF{y})$ varies across realizations of $\BF{y}$, which brings the difficulty of deriving the distribution function. Note that HVI on $\R^m$ is actually a piecewise-defined function. It suffices to first identify the set on which the restriction of $\hvi^{+}$ admits a fixed expression and then derive the conditional distribution function of HVI on such a set. As the first step, we provide a characterization of such a set.
\begin{lemma} \label{lm:piecewise}
Given a Pareto approximation set $\pfa \subset \R^m$ and a compact and connected set $\mathrm{S} \subset \R^m$ that dominates $\pfa$. If every point in $S$ dominates the same subset of $\pfa$, then the restriction $\hvi^{+}|_{S}$ is not a piecewise function and is continuous.
\end{lemma}
\begin{proof}
Let $U=\{\BF{p}\in\pfa\colon\forall\BF{s}\in S, \BF{s} \prec \BF{p}\}$. Based on the assumption that every point in $S$ dominates the same subset $U$ of $\pfa$, we could reformulate for every point $\BF{s}$, its hypervolume improvement, namely $\hvi^{+}(\BF{s})=\hv((\pfa\setminus U) \cup \{\BF{s}\}) - \hv(\pfa)$. Since $\pfa\setminus U$ and $\{\BF{s}\}$ are mutually non-dominated by construction, $\hv((\pfa\setminus U) \cup \{\BF{s}\})$ is not piecewise-defined and continuous.
\end{proof}
\begin{remark}
Obviously, the largest sets satisfying Lemma~\ref{lm:piecewise} are the hyperboxes constructed by the intersection of the $m-1$ dimensional hyperplanes passing through each point in \pfa. It is easy to verify that the above lemma also applies to the negative HVI. For the bi-objective case ($m=2$), we show an example of those cells in Fig.~\ref{fig:2dhv}. 
\end{remark}

In the paper, we focus on the bi-objective case, and we use the following notations for convenience: $\forall \BF{x}\in\domain$, $\mu_1 = \hat{f}_1(\BF{x}),\mu_2 = \hat{f}_2(\BF{x})$, and $\sigma_1^2 = \hat{k}_1(\BF{x}, \BF{x}),\sigma_2^2 = \hat{k}_2(\BF{x}, \BF{x})$. 

\paragraph{Cell partition of the objective space}
In bi-objective scenarios, we partition the objective space w.r.t.~the approximation points. For a finite approximation set $\mathcal{P}$ of $n$ points, we consider the extended real line $\mathbb{R}\cup \{-\infty,\infty\}$ and use it to augment \pfa with two extreme points $(-\infty, r_2)^\top$ and $(r_1, -\infty)^\top$, i.e., $\widetilde{\pfa}=\pfa \cup \{ (r_1, -\infty), (-\infty, r_2)\}$. 
Without loss of generality, we index the approximation points of $\widetilde{\pfa}$ in the increasing order w.r.t.~their first objective values, i.e., $\vq^{(0)}, \vq^{(1)}, \cdots, \vq^{(n+1)}$ where $-\infty=y^{(0)}_1 < y^{(1)}_1 < \cdots < y^{(n+1)}_1=r_1$. We denote by
$$
C(i,j) = \left[y_1^{(i)}, y_1^{(i+1)}\right] \times \left[y_2^{(n-j+1)}, y_2^{(n-j)}\right], i,j\in[0..n],
$$
the cell area bounded by axis-parallel lines that pass through the points in \pfa. For instance, in Fig.~\ref{fig:2dhv}, where $n=3$, $C(0,1)=[y_1^{(0)}, y_1^{(1)}]\times[y_2^{(3)}, y_2^{(2)}]$. In the following discussion, we will use the minimum $\vl^{(i,j)}$ and maximum $\vu^{(i,j)}$ of $C(i,j)$, namely,
$$
    \vl^{(i,j)} = \left(y_1^{(i)}, y_2^{(n-j+1)}\right)^\top, \, \vu^{(i,j)} = \left(y_1^{(i+1)}, y_2^{(n-j)}\right)^\top.
$$
In this manner, the entire objective space is partitioned into $(n+1)^2$ cells: $[-\BF{\infty}, \BF{r}]=\cup_{i,j\in[0..n]} C(i,j)$. Note that, for $i+j\leq n$, the union of the cells $C(i,j)$ represents the subset that is not dominated by $\widetilde{\pfa}$, i.e., $\mbox{ndom}(\widetilde{\pfa}) = \cup_{i+j\leq n} C(i,j)$ while for $i+j> n$, it indicates the subset dominated by $\widetilde{\pfa}$.

\paragraph{Conditional probability density function}
Taking the cell decomposition, we can express the distribution functions of generalized HVI by marginalizing the conditional distributions over cells:
\begin{align} \label{eq:full-distribution}
     & F_{\hvi(\BF{y})\mid \E}(\delta) =\sum_{i,j\in[0..n]}F_{\hvi(\BF{y})}^{(i, j)}(\delta) \Pr\left(\BF{y} \in C(i,j) \mid \E, \BF{x}\right), \nonumber \\
     &= \sum_{i,j\in[0..n]}F_{\hvi(\BF{y})}^{(i, j)}(\delta) \left[\Phi_{\mu_1, \sigma_1}\left(u_1^{(i,j)}\right) - \Phi_{\mu_1, \sigma_1}\left(l_1^{(i,j)}\right)\right] \nonumber \\
     & \times\left[\Phi_{\mu_2, \sigma_2}\left(u_2^{(i,j)}\right) - \Phi_{\mu_2, \sigma_2}\left(l_2^{(i,j)}\right)\right], 
\end{align}
where $F_{\hvi(\BF{y})}^{(i, j)}\in \{\PDF_{\hvi(\BF{y})}^{(i, j)}, \CDF_{\hvi(\BF{y})}^{(i, j)}\}$ denotes either the cumulative distribution function (CDF) or the probability density function (PDF) of HVI conditioned on cell $C(i,j)$. 
Note that within a non-dominated cell $C(i,j)$ ($i+j\leq n$), the conditional HVI takes it value in $\left[\hvi^{+}(\BF{u}^{(i,j)}), \hvi^{+}(\BF{l}^{(i,j)})\right]$. Similarly, within dominated cell $C(i,j)$ ($i+j > n$), the HVI's range is $\left[\hvi^{-}(\BF{u}^{(i,j)}),\hvi^{-}(\BF{l}^{(i,j)})\right]$.
We shall proceed to derive the conditional PDF as follows.

\begin{theorem}\label{thm:HVI-decomposition}
Assume the above cell partition of the objective space w.r.t. augmented approximation set $\widetilde{\pfa}$. For $0\leq i+j\leq n$, where $i,j \in [0..n]$, the hypervolume improvement of $\mathbf{y}=(y_1,y_2)^\top\in\mathbb{R}^2$ restricted to cell $C(i,j)$ can be expressed as follows:
\begin{align*}
    \Delta^{+}(\mathbf{y})|_{C(i,j)}&=
    \lambda(\greeninlinebox{\text{$A_1$}})
    +\lambda(\blueinlinebox{\text{$A_2$}})
    +\lambda(\orangeinlinebox{\text{$A_3$}})
    + \lambda(\redinlinebox{\text{$A_4$}}),
\end{align*}
where $\lambda$ is the Lebesgue measure on $\mathbb{R}^2$ and
\begin{align*}
    \greeninlinebox{\text{$A_1$}} &= \left[y_1, u_1^{(i, j)}\right] \times \left[y_2, u_2^{(i, j)}\right],  \quad\enskip \\
    \blueinlinebox{\text{$A_2$}} & =\left[u_1^{(i, j)}, y_1^{(n-j+1)}\right]\times\left[y_2, u_2^{(i, j)}\right],\\ \orangeinlinebox{\text{$A_3$}}&=\left[y_1, u_1^{(i, j)}\right]\times\left[u_2^{(i, j)}, y_2^{(i)}\right], \quad \\
    \redinlinebox{\text{$A_4$}} &=\operatorname{dom}\left(\left\{\mathbf{u}^{(i,j)}\right\}\right)\setminus\operatorname{dom}\left(\widetilde{\pfa}\right).
\end{align*}
\end{theorem}
\begin{proof}
An illustration has been given in Fig.~\ref{fig:2dhv}. Let $S=\operatorname{dom}(\{\BF{y}\})\setminus\operatorname{dom}(\widetilde{\pfa})$ denote the set that is dominated by $\BF{y}$ but not by $\widetilde{\pfa}$. 
We first consider the set $A_4$, the subset dominated by the maximum point $\BF{u}^{(i,j)}$, which is clearly contained in $S$ as $\BF{y}|_{C(i,j)}\preceq\BF{u}^{(i,j)}$. Note that, when $\BF{y}$ is restricted to $C(i,j)$, its projection along $f_1$ onto the attainment boundary is always $(y_1^{(n-j+1)}, y_2)^\top$; its projection along $f_2$ onto the attainment boundary is always $(y_1, y_2^{(i)})^\top$. Hence, we can express the reminder set $S\setminus A_4=[y_1, y_1^{(n-j+1)}]\times[y_2, u_2^{(i,j)}] + [y_1, u_1^{(i,j)}]\times[y_2, y_2^{(i)}] - [y_1, u_1^{(i, j)}] \times [y_2, u_2^{(i, j)}]= A_1 + A_2 + A_3$.
\end{proof}
Based on Thm.~\ref{thm:HVI-decomposition}, we can express the HVI restricted to cell $C(i,j)$ as:
\begin{align*}
 \Delta^{+} (\mathbf{y})|_{C(i,j)} & =  (u_1^{(i,j)} - y_1)(u_2^{(i,j)} - y_2)  \\
& + (y_1^{(n-j+1)} - u_1^{(i,j)})(u_2^{(i,j)} - y_2) \\
& + (u_1^{(i,j)} - y_1)(y_2^{(i)} - u_2^{(i,j)}) + \hvi^{+}(\vu^{(i,j)})\\
& =z_1z_2 + \gamma^{(i,j)},
\end{align*}
where
$
z_1 = y_1^{(n+1-j)} - y_1, z_2 = (y_2^{(i)}-y_2), \gamma^{(i,j)}=\hvi^{+}(\vu^{(i,j)}) - (y_1^{(n+1-j)} - u_1^{(i,j)}).
$
Note that $z_1\sim\mathcal{N}(y_1^{(n-j+1)} - \mu_1, \sigma_1^2)$ and $z_2\sim\mathcal{N}(y_2^{(i)} - \mu_2, \sigma_2^2)$ are Gaussian random variables truncated to $[L_1, U_1] \coloneqq [y_1^{(n-j+1)} - u_1^{(i,j)}, y_1^{(n-j+1)} - l_1^{(i,j)}]$ and $[L_2, U_2] \coloneqq [y_2^{(i)} - u_2^{(i,j)}, y_2^{(i)} - l_2^{(i,j)}]$, respectively. Thereby, the distribution of $\Delta^{+}(\mathbf{y})|_{C(i,j)}$ can be expressed via that of products of truncated Gaussians.

Let $\mu_1' = y_1^{(n+1-j)} - \mu_1$, $\mu_2' = y_2^{(i)} - \mu_2$. We have the following expression of HVI's distribution for the non-dominated cells:
$\forall \delta \in \left[\hvi^{+}(\BF{u}^{(i,j)}), \hvi^{+}(\BF{l}^{(i,j)})\right]$,
\begin{align}
    & \PDF_{\hvi^{+}(\mathbf{y})}^{(i,j)}(\delta)  =\PDF_{z_1z_2}(\delta - \gamma^{(i,j)}) \nonumber \\ 
    & \qquad =D_1D_2\int_{\alpha(p)}^{\beta(p)}\phi_{\mu_1',\sigma_1}(x)\phi_{\mu_2',\sigma_2}\left(\frac{p}{x}\right)x^{-1}\ud x, \label{eq:conditional-pdf} \\
    &p= \delta - \gamma^{(i,j)},\quad \nonumber \\
    &D_1=[\Phi_{\mu_1', \sigma_1}(U_1) - \Phi_{\mu_1', \sigma_1}(L_1)]^{-1}, \quad \nonumber \\ 
    & D_2 =[\Phi_{\mu_2', \sigma_2}(U_2) - \Phi_{\mu_2', \sigma_2}(L_2)]^{-1}, \nonumber
\end{align}
where $\phi_{\mu,\sigma}$ and $\Phi_{\mu,\sigma}$ denote the PDF and CDF of a Gaussian random variable with mean $\mu$ and standard deviation $\sigma$, respectively. The integration bounds are determined as follows. If $L_1U_2 < U_1L_2$~\footnote{When $L_1U_2 > U_1L_2$, it suffices to swap variables $z_1$ and $z_2$, and apply Eq.~\eqref{eq:integral-bound}. }, we define:
\begin{equation} \label{eq:integral-bound}
\!\!\![\alpha(p), \beta(p)] = 
\begin{cases}
[L_1, p/L_2],  &  L_1L_2 \leq p < L_1U_2 \\ 
[p / U_2, p / L_2], & L_1U_2 \leq p < U_1L_2 \\
[p / U_2, U_1], &  U_1L_2 \leq p \leq U_1U_2
\end{cases}
\end{equation}
For a cell in the dominated space, i.e., $C(i,j)$ with $i+j>n$, we could invert the coordinate system to treat it in the same way as the non-dominated part (notice that $\BF{l}^{(i, j)}$ and $\BF{u}^{(i, j)}$ are swamped after inversion). Namely, we take the following inverted quantities: 
\begin{equation}\label{eq:inverted-quantities}
\begin{aligned}
&[L_1^{-}, U_1^{-}] \coloneqq [-U_1, -L_1], \quad \nonumber \\
& [L_2^{-}, U_2^{-}] \coloneqq [-U_2, -L_2], \nonumber \\
&\mu_1^{-} \coloneqq -\mu_1', \quad \mu_2^{-} \coloneqq -\mu_2', \quad \\
&
\gamma^{(i,j)} = -\Delta^-(\BF{l}^{(i,j)}) - (l_1^{(i,j)} - y_1^{(n+1-j)}).
\end{aligned}
\end{equation}
It is straightforward to verify that the above derivation still holds for the inverted quantities. In this case, the conditional density function can be computed with Eq.~\eqref{eq:conditional-pdf}: $\forall\delta\in[\Delta^{-}(\BF{u}^{(i,j)}),\Delta^{-}(\BF{l}^{(i,j)})]$,
$$
\PDF_{\hvi^{-}(\mathbf{y})}^{(i,j)}(\delta)=\PDF_{\hvi^{+}(\mathbf{y})}^{(i,j)}(-\delta)=\PDF_{z_1z_2}(-\delta - \gamma^{(i,j)}),
$$
where we substitute the quantities (e.g., $L_1$) with the inverted ones (e.g., $L_1^{-}$) in $\PDF_{z_1z_2}$.

\paragraph{Conditional cumulative distribution function} Taking the conditional density function, the cumulative distribution of the hypervolume improvement can be derived for a non-dominated cell $C(i, j)$ ($i + j \leq n$). For $\delta \in \left[\hvi^{+}(\BF{u}^{(i,j)}), \hvi^{+}(\BF{l}^{(i,j)})\right]$ and $p= \delta - \gamma^{(i,j)}$, we have:
\begin{align}
    &\CDF_{\hvi^{+}(\BF{y})}^{(i,j)}(\delta)  =\int_{L_1L_2}^{p} \PDF_{z_1z_2}\left(x\right) \ud x \nonumber\\
    & =D_1D_2\left(G +\int_{\alpha(p)}^{\beta(p)}\phi_{\mu_1',\sigma_1}(\zeta)\Phi_{\mu_2',\sigma_2}\left(\frac{p}{x}\right)\ud x\right), \label{eq:conditional-cdf}\\ 
    &G=\Phi_{\mu_2',\sigma_2}\left(U_2\right)\left[ \Phi_{\mu_1',\sigma_1}\left(\alpha(p)\right) - \Phi_{\mu_1',\sigma_1}\left(L_1\right)\right] + \nonumber \\
    & \quad \quad \Phi_{\mu_2',\sigma_2}\left(L_2\right)\left[ \Phi_{\mu_1',\sigma_1}\left(L_1\right) - \Phi_{\mu_1',\sigma_1}\left(\beta(p)\right)\right], \nonumber
\end{align}
where the integration bounds $\alpha, \beta$ are defined in Eq.~\eqref{eq:integral-bound}.
For a cell $C(i,j)$ in the dominated space ($i+j>n$), we take the same trick of inverting the coordinate system as above: $\forall \delta \in [\hvi^{-}(\BF{u}^{(i,j)}), \hvi^{-}(\BF{l}^{(i,j)})]$,
$$
\CDF_{\hvi^{-}(\BF{y})}^{(i,j)}(\delta) = 1 - \CDF_{\hvi^{+}(\BF{y})}^{(i,j)}(-\delta),
$$
where we have to take the inverted quantities in Eq.~\eqref{eq:inverted-quantities}.
\begin{figure*}[t]
    \begin{minipage}[c]{0.5\textwidth}
        \includegraphics[width=\textwidth,trim=4mm 2mm 15mm 10mm,clip]{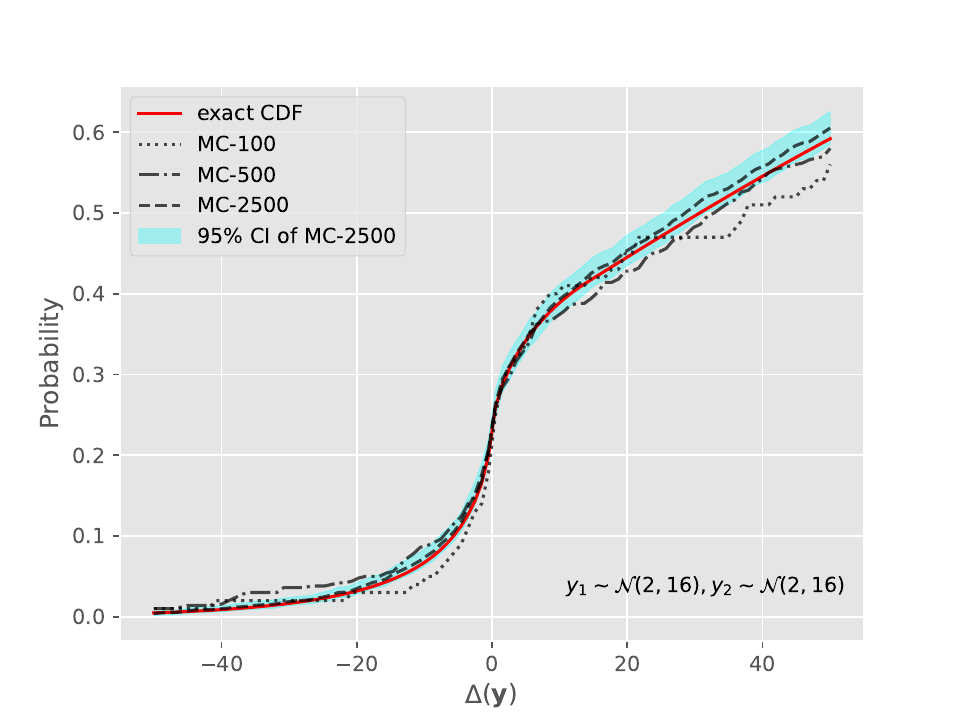}
    \end{minipage}
    \begin{minipage}[c]{0.5\textwidth}
        \includegraphics[width=\textwidth,trim=3mm 2mm 15mm 10mm,clip]{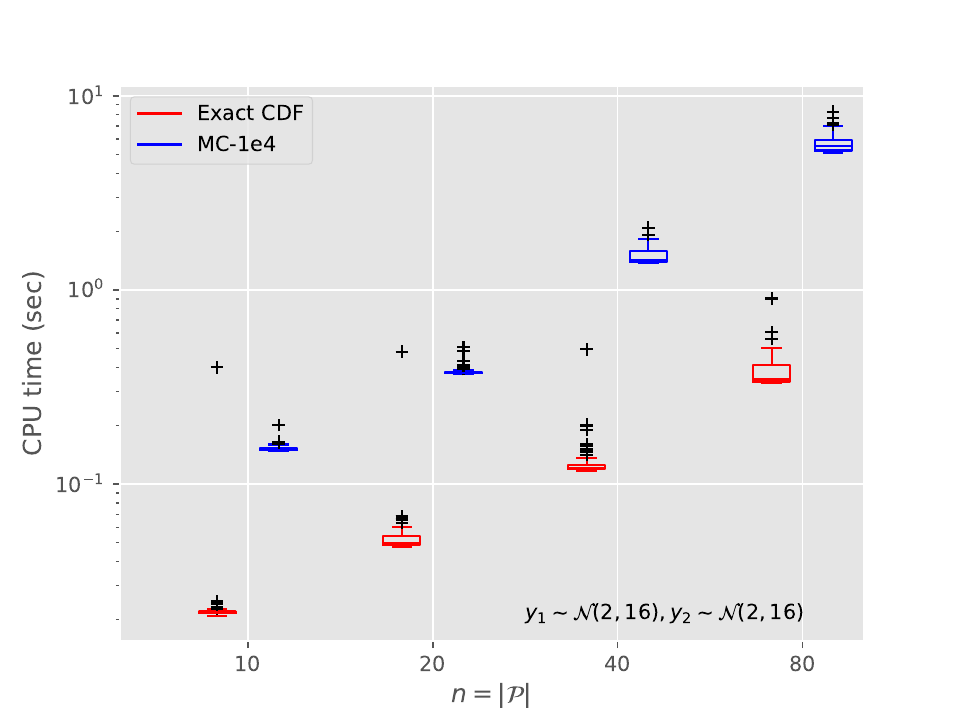}
    \end{minipage}
\caption{\label{fig:example}\textbf{Left}: For the Pareto front in Fig.~\ref{fig:2dhv}, we show the CDF of $\BF{y}$ computed from the exact and MC methods (using 100, 500, and 2\,500 samples). \textbf{Right}: The CPU time for the exact and the MC method (with $10^4$ sample points) w.r.t.~an increasing number of points of \pfa.}
\end{figure*}

\paragraph{Numerical computation}
We use the numerical integration\footnote{We employed the 21-point Gauss–Kronrod quadrature method with maximally $50$ sub-intervals.} to compute the distribution function in each cell and set the absolute error of the integration to $10^{-8}$. The time complexity of the PDF and CDF of HVI is quadratic w.r.t. the number of approximation points in \pfa due to the quadratic number terms in Eq.~\eqref{eq:full-distribution}. To reduce the time complexity, we propose to prune the computation on some cells: (1) the ones on which the probability mass of $\BF{y}$ is sufficiently small. We only include the cells that overlap with the range of $\mu\pm3\sigma$ in the computation; (2) if the value of HVI to compute is out of the range of conditional HVI on a cell. In the left plot of Fig~\ref{fig:example}, we have illustrated an example of the CDF of HVI computed from both the exact distribution and the Monte Carlo (MC) method.
It is necessary to compare the computational cost of the exact distribution to that of the MC method to approximate the cumulative distribution function. Generally, for achieving an accuracy of $\tau$, the numerical integration requires $O(\tau^{-1})$~\cite{Novak14}, resulting in an overall complexity of $O(\tau^{-1}(n+1)^2)$ for the exact method. In contrast, the MC method requires sampling $O(\tau^{-2})$ realizations of $\BF{y}$ and calculating the hypervolume improvement thereof\footnote{It requires computing the hypervolume of the Pareto front approximation, which has a time complexity of $\Theta(n\log n)$ when $m=2,3$~\cite{BeumeFLPV09}.}, giving rise to the complexity of $O(\tau^{-2}n\ln n)$. Also, in the right plot of Fig.~\ref{fig:example}, we compare the CPU time of the exact expression to the MC computation when varying the number of points in the Pareto front, which shows that under a comparable numerical accuracy, the CPU time consumed by the exact computation is roughly one order of magnitude lower than that of the MC method, for a wide range of the cardinality of \pfa.

\section{$\varepsilon$-Probability of Hypervolume Improvement} \label{sec:acquisition}
In this section, we leverage the commonly used $\varepsilon$-PoI function with HVI's distribution in order to demonstrate the usefulness thereof in Bayesian optimization. 
The $\varepsilon$-PoI function translates the mean of posterior distribution towards \pfa by $\varepsilon\BF{1}_m$ and computes the probability of improving the current \pfa (see its definition in Sec.~\ref{sec:background}).
When computing $\varepsilon$-PoI, all objective points that are taken into account have a minimal distance of $\varepsilon$ to the attainment boundary of \pfa. Despite its simplicity, we find it difficult to relate $\varepsilon$ to the quantiles of HVI's distribution: from the posterior distribution, we can generate two sample points that both have $\varepsilon$ minimal distance to the attainment boundary but differ hugely in their hypervolume improvements. 
\begin{figure*}[t]
    \centering
    \includegraphics[width=\textwidth, trim=4mm 0mm 0mm 0mm, clip]{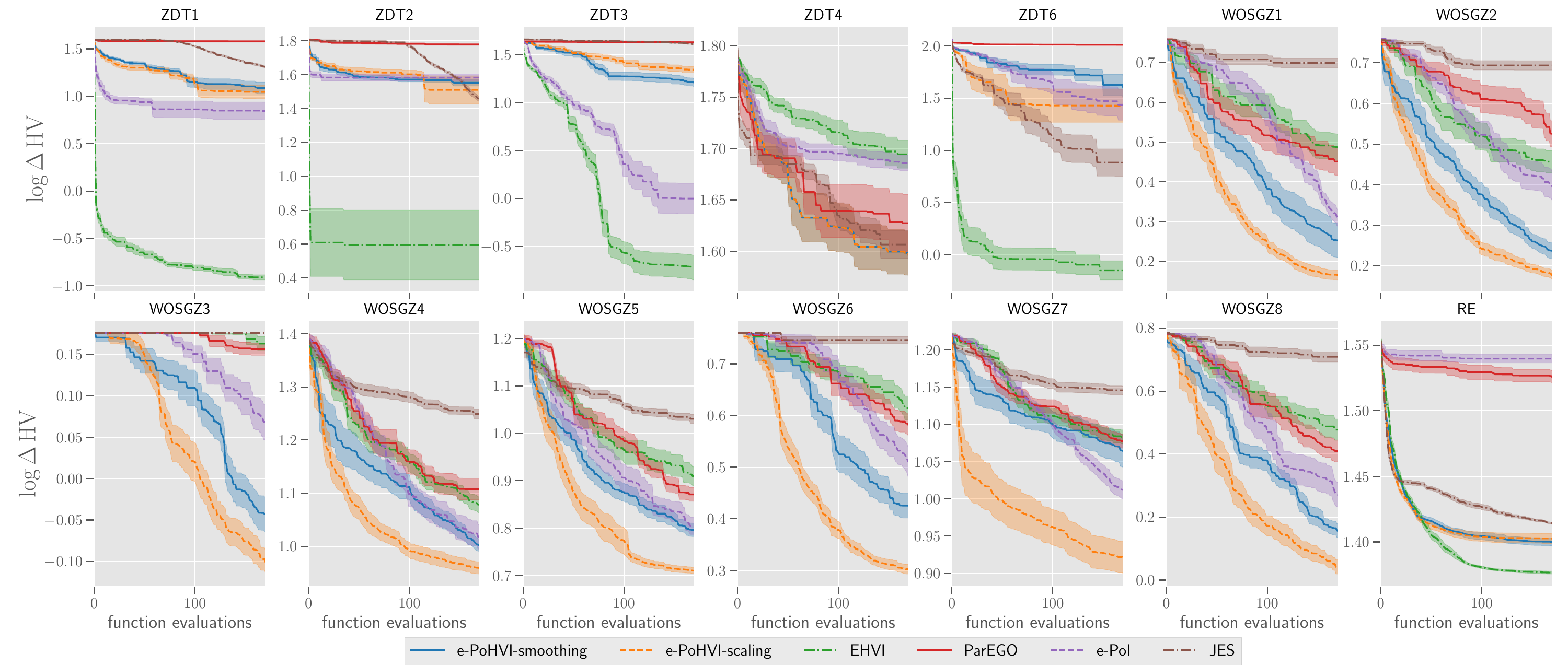}
    \caption{The log difference between the hypervolume of the best-so-far approximation set \pfa and the target hypervolume over function evaluations. The target hypervolume is obtained with 1\,000 points evenly sampled on the Pareto front of each problem. We show the mean and standard error of the log differences measured from $15$ independent runs of each acquisition function on each problem.}
    \label{fig:convergence}
\end{figure*}

To mitigate this issue, we propose the \emph{$\varepsilon$-Probability of Hypervolume Improvement} ($\varepsilon$-PoHVI) function, which computes the probability of making at least $\varepsilon$ hypervolume improvement to \pfa. This probability can be computed directly by HVI's CDF function defined in Eq.~\eqref{eq:conditional-cdf}:
\begin{equation}
    \operatorname{\varepsilon-PoHVI}(\BF{x};\pfa, \BF{r},\varepsilon) = 1-\CDF_{\hvi(\BF{y})\mid \E,\BF{x}}\left(\varepsilon\right),
\end{equation}
where $\varepsilon$ can be considered a lower quantile of HVI's distribution. This definition gives rise to two advantages: (1) When HVI exhibits a huge dispersion, $\varepsilon$-PoHVI is still safe to use, compared to EHVI, which would become less meaningful in this circumstance; (2) the user of Bayesian optimization can precisely control the level of the minimal improvement in each iteration. Compared to $\varepsilon$-PoI (time complexity: $\Theta(n\log n)$), our new acquisition function suffers from relatively higher computation overheads (quadratic; see ``Numerical computation'' above), and it requires the user to specify a reference point (for computing the hypervolume improvement), which is not needed in $\varepsilon$-PoI. Nevertheless, we are interested in whether the losses in the computation cost of $\varepsilon$-PoHVI can speed up the empirical convergence of Bayesian optimization on challenging problems. In practice, the free parameter $\varepsilon$ in those two acquisition functions is either manually determined or tuned via hyperparameter tuning. In this paper, we also propose two control schemes:
\begin{itemize}
    \item \emph{$\varepsilon$-PoHVI-scaling} determines the parameter $\varepsilon_{t}$ at iteration $t$ with the schedule: $\varepsilon_{t} = \varepsilon_0 \exp{(-ct)}$,
    where $\varepsilon_0 = 0.05, c=0.02$. We include, in the supplementary material, a rule-of-thumb to set the hyperparameter $\varepsilon_0$ and $c$.
    This schedule is motivated by the fact that when converging to the Pareto front, the steps of each approximation point tend to decrease.
    \item \emph{$\varepsilon$-PoHVI-smoothing} exponentially smooths of the hypervolume improvement measured in the optimization: $\varepsilon_{t+1} = \alpha\left(\hv(\pfa_{t}, \refp) - \hv(\pfa_{t-1},\refp)\right) + (1-\alpha) \varepsilon_{t}$,
    where $\alpha=0.5$ and $\varepsilon_0 = 0.05$. The intuition is to take the hypervolume improvement realized in the optimization history to set the $\varepsilon$ value for the next iteration. 
\end{itemize}

\section{Experiments}\label{sec:experiments}
\begin{table*}[t]
\centering
\caption{Over all test problems, we perform the pairwise Wilcoxon's Rank-Sum test matrix at significance level 0.05 with p-value correction, where $+/\approx/-$ indicates that the algorithm in the column is significantly better/not different/worst than the ones in rows.}
\label{tab:sumRankSum_all}
\begin{tabular}{l|cccc}
\toprule
                           & \textbf{e-PoI} & \textbf{EHVI} & \textbf{e-PoHVI-smoothing} & \textbf{e-PoHVI-scaling} \\
\hline
\textbf{e-PoI}             & n.a.         & 5/2/7         & 6/6/2                      & 9/3/2                    \\
\textbf{EHVI}              & 7/2/5          & n.a.         & 7/2/5                      & 8/1/5                    \\
\textbf{e-PoHVI-smoothing} & 2/6/6          & 5/2/7         & n.a.                     & 8/5/1                    \\
\textbf{e-PoHVI-scaling}   & 2/3/9          & 5/1/8         & 1/6/7                      & n.a.                   \\
\hline
Sum of $+/\approx/-$                  & 11/11/20       & 15/5/22       & 14/14/14                   & \textbf{25/9/8}    \\
\bottomrule
\end{tabular}
\end{table*}

\paragraph{Experimental setup} 
We investigate the empirical performance of $\varepsilon$-PoHVI against $\varepsilon$-PoI and EHVI on three sets of test problems: (1) the classical bi-objective ZDT problems~\cite{DBLP:journals/ec/ZitzlerDT00}, which have regular-shaped Pareto front (either convex or concave). We selected problems ZDT1-4 and 6 (ZDT5 is a discrete optimization problem); (2) WOSGZ1-8 problems~\cite{wosgz} whose Pareto fronts are more difficult to approximate (WOSGZ9-16 are tri-objective problems); (3) a real-world problem - \emph{four bar truss design}~\cite{re1,IshibuchiREproblems2020} (denoted as RE), which has a convex Pareto front and the ranges of two objective functions differ drastically. The decision space is set to $[0, 1]^{30}$ for ZDT1-3 and ZDT6, $[0,1] \times [-5,5]^{29}$ for ZDT4, $[0,1]\times[-1, 1]^{29}$ for all WOSGZ problems, and $[1, 3] \times [1, \sqrt{2}] \times [1, \sqrt{2}] \times [1, 3]$ for RE. In the objective space, we take the reference point $\mathbf{r}=[15,15]$ as recommended in~\cite{DBLP:journals/ec/ZitzlerDT00} for ZDT problems when computing the hypervolume. The reference points are set to $[1.2, 1.2]$ and $[3\,000,0.0383]$ for the WOSGZ and the RE problem, respectively, as suggested in~\cite{wosgz,dgemo}.\\
We implement $\varepsilon$-PoI, $\varepsilon$-PoHVI, EHVI in the DGEMO~\cite{dgemo} algorithmic framework\footnote{Our source code is available at \url{https://github.com/wangronin/HVI-distribution}} and test each of acquisition functions with $15$ independent runs on each test problem. Also, we compare PoHVI with two commonly-used multi-objective Bayesian optimization algorithms: Joint Entropy Search (JES)~\cite{TuGKS22} and ParEGO~\cite{Knowles06}. \\
We initialize the BO algorithm with $\min(60, 6d)$ points generated with Latin Hypercube sampling and terminate the algorithm at $170$ iterations. We build two independent Gaussian processes for each objective with Mat\'ern 5/2 kernel. We maximize the acquisition function in each iteration with covariance matrix adaptation evolution strategy (CMA-ES) algorithm~\cite{Hansen06}. Also, we take the above ``scaling'' control scheme to set the free parameter of $\varepsilon$-PoI.

\paragraph{Results and discussion} 
In Fig.~\ref{fig:convergence}, we show the mean convergence and 95\% confidence interval of the best-so-far hypervolume value of \pfa for BO equipped with different bi-objective acquisition functions. Overall, we see (1) two versions of $\varepsilon$-PoHVI outperform $\varepsilon$-PoI and EHVI substantially on WOSGZ problems while EHVI takes the lead on ZDT and RE problems; (2) $\varepsilon$-PoHVI outperforms the classic ParEGO algorithm across all functions; (3) the Joint Entropy search (JES) only outperforms PoHVI on ZDT6. \\
Compared to WOSGZ, ZDT problems are considered relatively easier to solve since (1) the Pareto fronts are very regular - either convex or concave; (2) there are no local Pareto fronts (except ZDT4); (3) on the Pareto front, the optimal distribution~\cite{AugerBB10} of points (w.r.t.~HV) is mostly uniform. Similarly, the real-world problem RE has a much lower search dimensionality (four) with a convex Pareto front. WOSGZ problems are, however, designed to have more realistic Pareto fronts (non-convex/non-concave with non-uniform optimal distribution), which is shown to be challenging to solve or model~\cite{wosgz}. \\
With a small data set, we expect a much higher GP's prediction uncertainty on WOSGZ problems than ZDTs, which is validated with numerical observations: in Table~\ref{tab:gp_sigma_zdt_wosgz}, we list GP's prediction uncertainty (posterior standard deviations) for each objective function over decision points sampled u.a.r.~from the decision space. We see, for instance, on WOSGZ1, the posterior standard deviation of the first objective is about $\sigma_1 = 1.5524$, which is orders of magnitude higher than that on ZDT1 ($\sigma_1 =0.00837$).
As a result, HVI's distribution on WOSGZs exhibits a high dispersion, and the EHVI function becomes less characteristic of the distribution. In contrast, $\varepsilon$-PoHVI is not affected by the large dispersion since it utilizes the quantile of HVI's distribution. Similarly, $\varepsilon$-PoI also suffers less from the high dispersion of HVI despite not connecting to the quantiles of HVI's distribution. \emph{We conclude that $\varepsilon$-PoHVI is advantageous when the prediction uncertainty of GP is high, which occurs if the objective functions are challenging to model with a small data set.} 

Between $\varepsilon$-PoI and $\varepsilon$-PoHVI, we see that $\varepsilon$-PoHVI (with both control schemes of $\varepsilon$) substantially improves the convergence speed and achieves better hypervolume values of the final approximation set \pfa on all test problems except ZDT1, 2, and 4. ZDT1 is one of the simplest bi-objective problems. On ZDT2, although $\varepsilon$-PoI shows a faster initial convergence, $\varepsilon$-PoHVI manages to hit about the same mean hypervolume value in the last few iterations. 
On ZDT4, all acquisition functions fail to get close to the Pareto front since this problem has many local Pareto fronts. Much more function evaluations are needed for all acquisition functions. Also, we see that the scaling control ($\varepsilon$-PoHVI-scaling) performs better than the exponential smoothing ($\varepsilon$-PoHVI-smoothing) on all problems except ZDT3. \\
Furthermore, we perform a pairwise Wilcoxon's Rank-Sum test (with $p$-value correction for multiple testing) on all test problems to verify the statistical significance of the result in the convergence chart. In table~\ref{tab:sumRankSum_all}, we show the test's outcome - each entry of the table ($+/\approx/-$) indicates out of 14 test problems, the number of cases where the algorithm in the first row is significantly better/not different/worst than the ones in the first column. Overall, $\varepsilon$-PoHVI-scaling outperforms both $\varepsilon$-PoI, $\varepsilon$-PoHVI-smoothing, and EHVI. Since $\varepsilon$-PoHVI-scaling is a comparable parameter control compared to that of $\varepsilon$-PoI. We can conclude that $\varepsilon$-PoHVI has better empirical performance than $\varepsilon$-PoI on the selected test problems. Finally, in the supplementary material, we have included the detailed descriptive statistics of the convergence of hypervolume values and the best, median, and worst $\pfa$ obtained by each method at the last iteration.

\section{Conclusions} \label{sec:conclusion}
We first propose a generalization to hypervolume improvement (HVI), which assigns nonzero value to the dominated points. Then, we derive the exact expression of the distribution functions of the hypervolume improvement (HVI) in the Bayesian optimization setup, in which we utilize a cell partition of the objective space.
Compared to the Monte Carlo approach, the numerical computation of the exact expression is computationally faster and numerically more accurate. Taking this distribution function, we propose a novel acquisition function, $\varepsilon$-Probability of Hypervolume Improvement, which shows a large empirical advantage over its counterparts. 

The limitation of this work is two-fold: (1) the distribution functions are derived in the bi-objective scenario with uncorrelated Gaussian processes for each objective. In practice, users of Bayesian optimization often wish to solve more than two objectives with a multi-output Gaussian process model. For correlated Gaussian posteriors, our analytical results (Eq.~\eqref{eq:conditional-pdf} and~\eqref{eq:conditional-cdf}) can be directly applied by substituting the probability density in Eq.~\eqref{eq:conditional-pdf} with the one of the correlated Gaussian. As for more than two objectives, the cell partition approach can be applied in principle. However, the exact formulation of HVI's distribution is very complicated to accommodate in this work, given the page limit.  
(2) the experimental study of the proposed acquisition function can be enhanced on more real-world problems, e.g., industrial optimization and hyper-parameter tuning tasks.

\section*{Acknowledgement}
The authors would like to
thank Michael Emmerich for various discussions and suggestions.
This work is supported by the Austrian Science Fund (FWF – Der Wissenschaftsfonds) under the project (I 5315, `ML Methods for Feature Identification Global Optimization).

\bibliography{references.bib}
\bibliographystyle{icml2024}

\newpage
\appendix
\onecolumn
\section{Appendix}
We illustrate how to determine the hyperparameters of $\varepsilon$-PoHVI-scaling: $\varepsilon_{t} = \varepsilon_0\exp{(-ct)}$.
We use the following reasoning to determine the hyperparameters in the scaling scheme. For example, on WOSGZ problems: (1) $\varepsilon_0$ controls the maximal HV improvement across all iterations; (2) we take $[1.2, 1.2]$ as the reference point, which gives rise to an upper bound of $1.44$ on the maximal HV value, provided that the ideal point on those problems is $[0, 0]$; (3) We assume that the initial HV value after random sampling is one-half the maximal HV value, i.e., $1.44 / 2$; (4) If the BO algorithm were to realize $\varepsilon_0 \exp{(-ct)}$ HV improvement in each iteration, then the total sum of all such improvements should be bounded above the maximal HV value to realize, i.e., $\varepsilon_0\sum_t \exp{(-ct)} \leq 1.44 / 2$. The hyperparameter of $\varepsilon$-PoHVI-smoothing function can be determined in a similar way.
\begin{figure}[h!]
    \centering
    \caption{On each test problem, we show some descriptive statistics: min, max, mean, median, standard deviation (std), and 25\%- and 75\%-quantiles of the hypervolume (HV) value observed at the last iteration of the BO algorithm. The entries are color-coded relative to the corresponding ones (e.g., we take all the mean values in a column) in the same column/problem, where a more greenish color indicates better performance and vice versa.}
    \includegraphics[width=\textwidth,trim=1.9cm 4cm 4cm 1.9cm,clip]{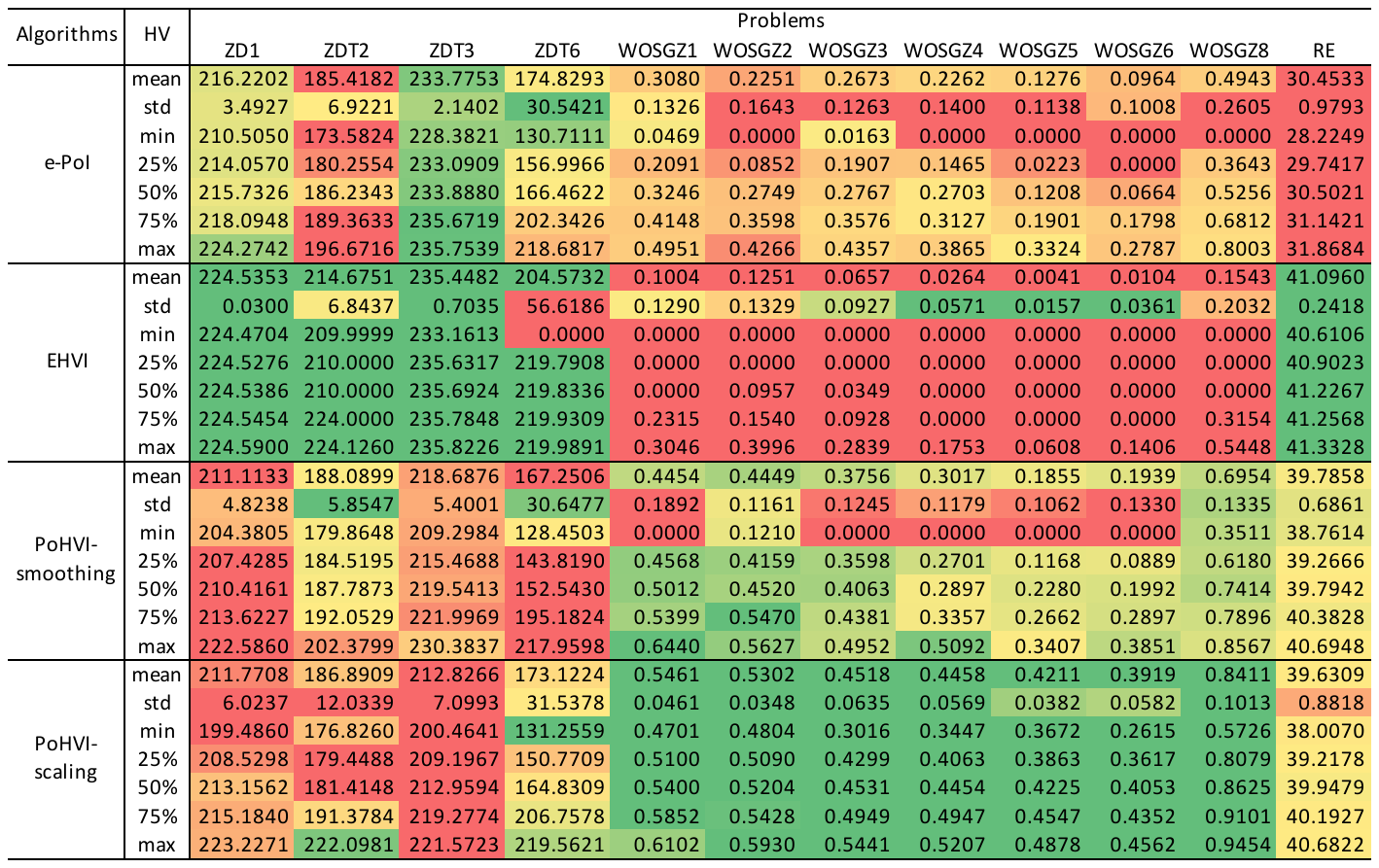}
    \label{fig:my_label}
\end{figure}
\vspace{-20pt}
\begin{table}[!h]
\centering
\caption{
GP's prediction uncertainty (posterior standard deviations $\sigma_1$ and $\sigma_2$ for each objective function, respectively) aggregated over uniformly sampled decision points on ZDTs (top) and WOSGZs (bottom) and BO's iterations. The average and standard error of the uncertainty is estimated over 15 repetitions of BO.}
\label{tab:gp_sigma_zdt_wosgz}
\resizebox{\columnwidth}{!}{%
\begin{tabular}{l|cc|cc|cc|cc|cc}
\toprule
     & \multicolumn{2}{c|}{ZDT1} & \multicolumn{2}{c|}{ZDT2} & \multicolumn{2}{c|}{ZDT3} & \multicolumn{2}{c|}{ZDT4} & \multicolumn{2}{c}{ZDT6} \\
     & $\sigma_1$  & $\sigma_2$ & $\sigma_1$   & $\sigma_2$  & $\sigma_1$   & $\sigma_2$  & $\sigma_1$   & $\sigma_2$  & $\sigma_1$   & $\sigma_2$  \\
\midrule
average & 0.008378    & 0.044345   & 0.008819     & 0.008107    & 0.003032     & 0.400788    & 0.023912     & 22.251792   & 0.043914     & 0.173658    \\
standard error  & 0.043334    & 0.015357   & 0.041234     & 0.003639    & 0.004248     & 0.028704    & 0.019340     & 0.8503408   & 0.030066     & 0.199384 \\
\bottomrule
\end{tabular}
}
\resizebox{\columnwidth}{!}{%
\begin{tabular}{l|cc|cc|cc|cc|cc|cc|cc|cc}
\toprule
     & \multicolumn{2}{c|}{WOSGZ1} & \multicolumn{2}{c|}{WOSGZ2} & \multicolumn{2}{c|}{WOSGZ3} & \multicolumn{2}{c|}{WOSGZ4} & \multicolumn{2}{c|}{WOSGZ5} & \multicolumn{2}{c|}{WOSGZ6} & \multicolumn{2}{c|}{WOSGZ7} & \multicolumn{2}{c}{WOSGZ8} \\
     & $\sigma_1$   & $\sigma_2$  & $\sigma_1$   & $\sigma_2$  & $\sigma_1$   & $\sigma_2$  & $\sigma_1$   & $\sigma_2$  & $\sigma_1$   & $\sigma_2$  & $\sigma_1$   & $\sigma_2$  & $\sigma_1$   & $\sigma_2$  & $\sigma_1$   & $\sigma_2$  \\
\midrule
average & 1.5524       & 0.0915      & 3.1203       & 0.1797      & 3.1615       & 0.1834      & 3.9904       & 0.2614      & 4.4331       & 0.3117      & 5.4329       & 0.3487      & 3.5918       & 0.2047      & 3.8567       & 0.2651      \\
standard error  & 0.5048       & 0.0352     & 0.9660       & 0.0654      & 1.0346       & 0.0690      & 1.3334       & 0.1030      & 1.4432       & 0.1061      & 1.7562       & 0.1255      & 0.8340       & 0.0628      & 1.0050       & 0.0861      \\
\bottomrule
\end{tabular}
}
\end{table}

\begin{figure}
     \centering
     \begin{subfigure}[b]{0.85\textwidth}
         \centering
         \includegraphics[width=\textwidth]{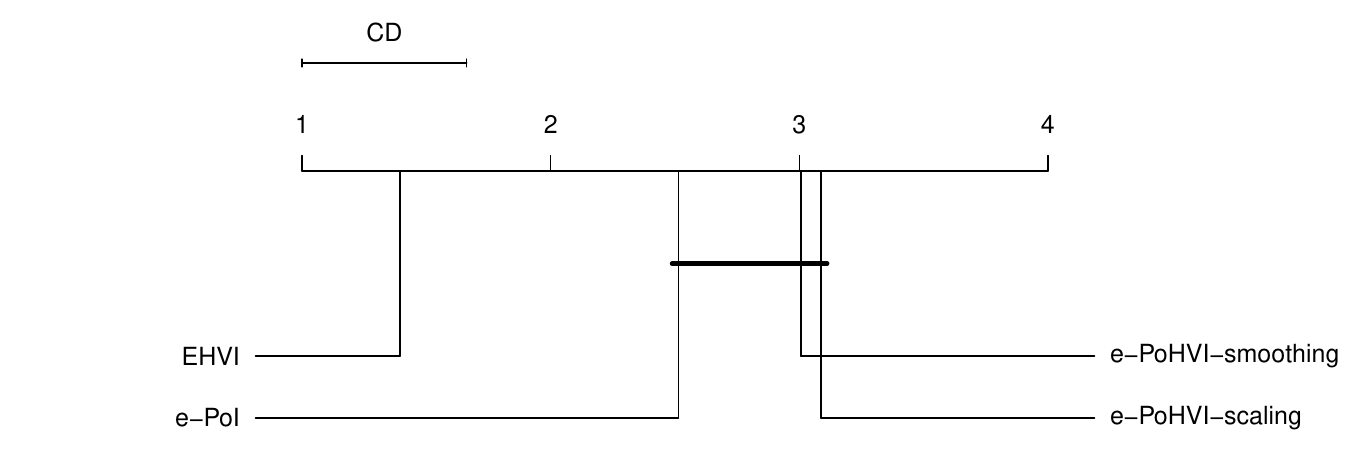}
         \caption{ZDTs}
     \end{subfigure}
     \begin{subfigure}[b]{0.85\textwidth}
         \centering
         \includegraphics[width=\textwidth]{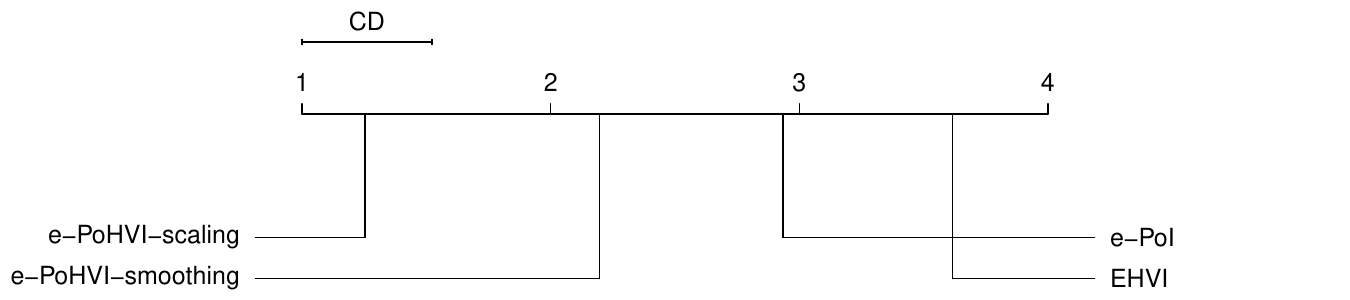}
         \caption{WOSGZs}
     \end{subfigure}
     \begin{subfigure}[b]{0.85\textwidth}
         \centering
         \includegraphics[width=\textwidth]{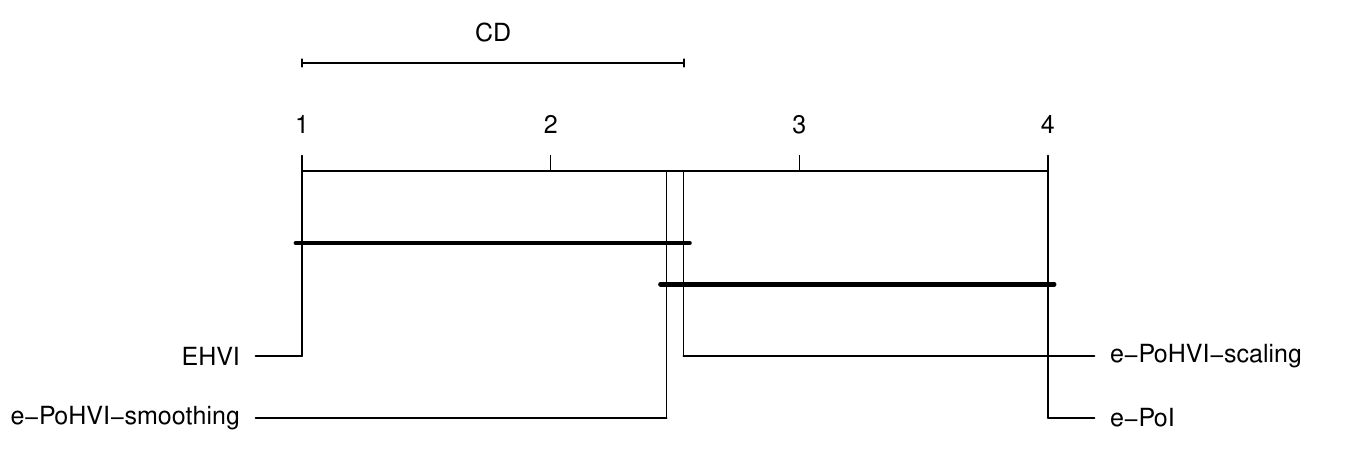}
         \caption{RE}
     \end{subfigure}
     \begin{subfigure}[b]{0.85\textwidth}
         \centering
         \includegraphics[width=\textwidth, trim=20mm 0mm 20mm 0mm,clip]{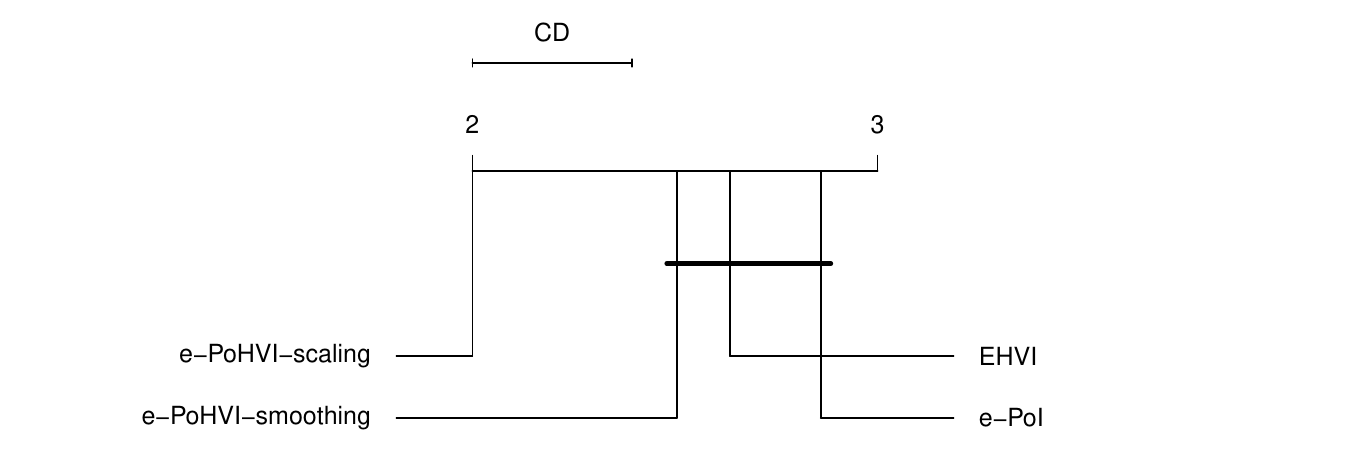}
         \caption{All problems}
     \end{subfigure}
      
    \caption{The critical difference (CD) chart obtained with the Nemenyi posthoc testing procedure to a Friedman test. The performance of two acquisition functions significantly differs on a problem set if their average ranks of HV values differ by at least the critical difference shown as the interval atop each chart. The thick horizontal line indicates a clique of acquisition functions with no significant difference.}
    \label{fig:CD-plots}
\end{figure}

\begin{figure}
     \centering
     \begin{subfigure}[b]{0.8\textwidth}
         \centering
         \includegraphics[width=\textwidth]{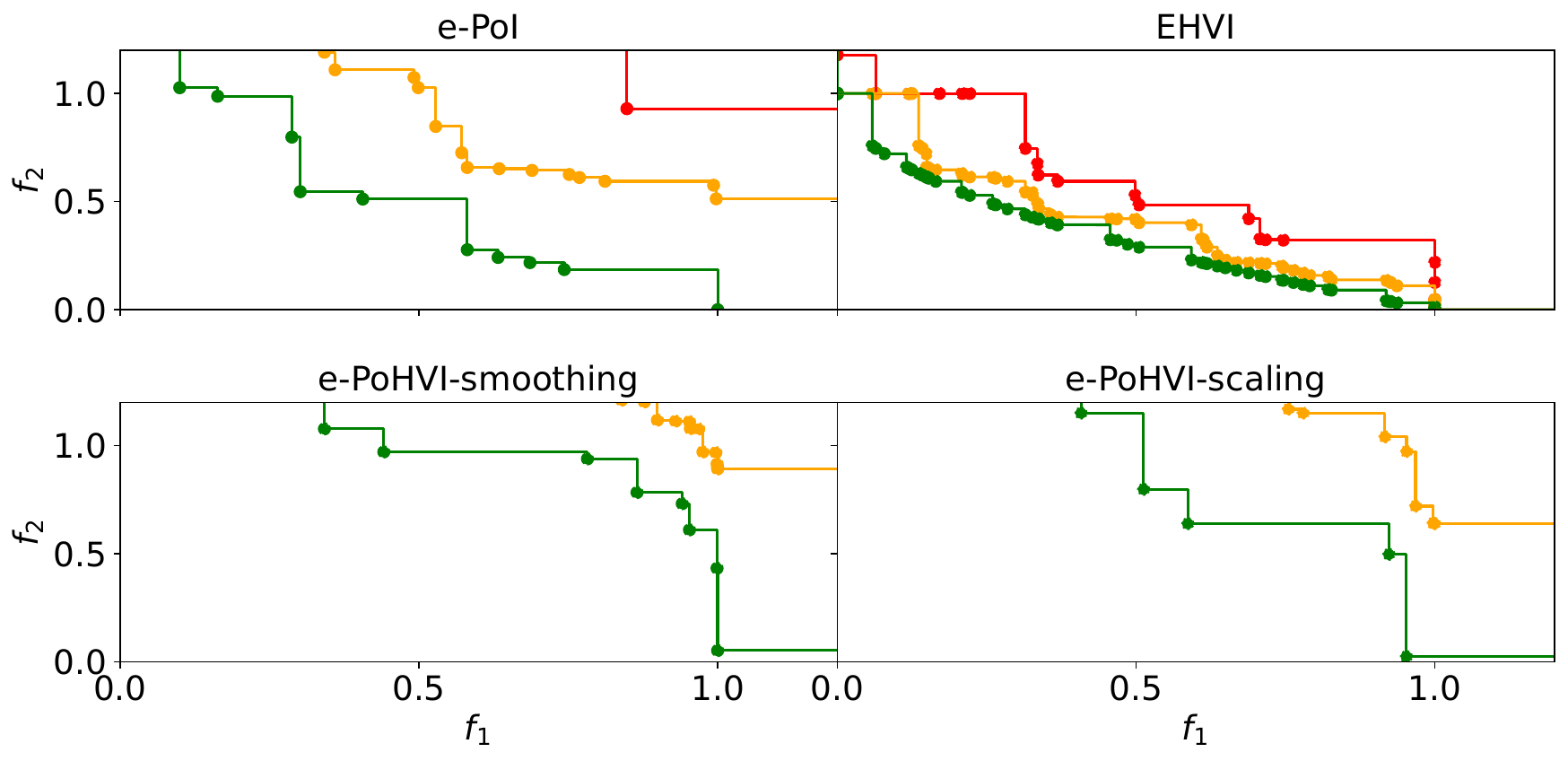}
         \caption{ZDT1}
     \end{subfigure}
     \hfill
     \begin{subfigure}[b]{0.8\textwidth}
         \centering
         \includegraphics[width=\textwidth]{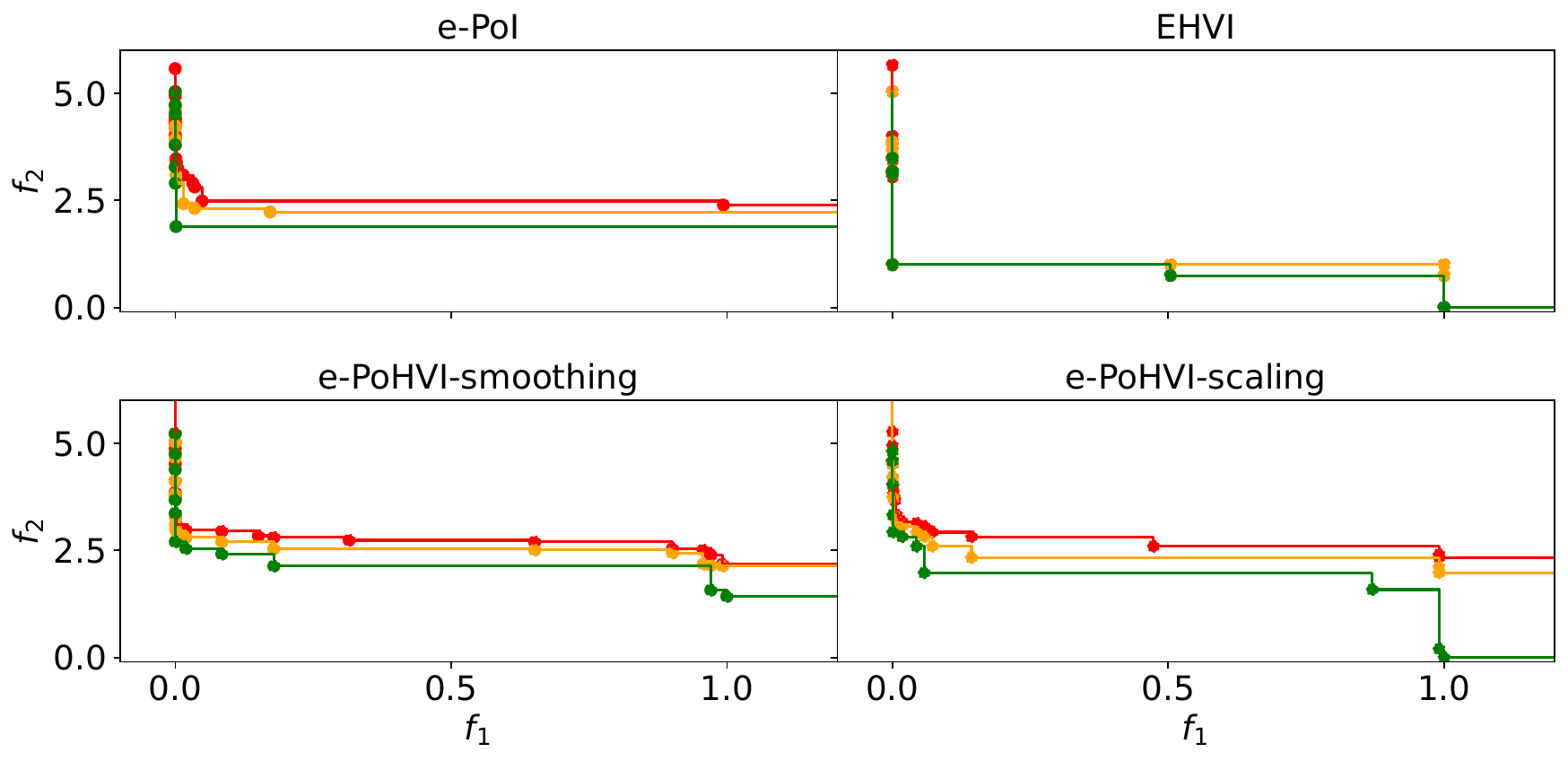}
         \caption{ZDT2}
     \end{subfigure}
      \hfill
     \begin{subfigure}[b]{0.8\textwidth}
         \centering
         \includegraphics[width=\textwidth]{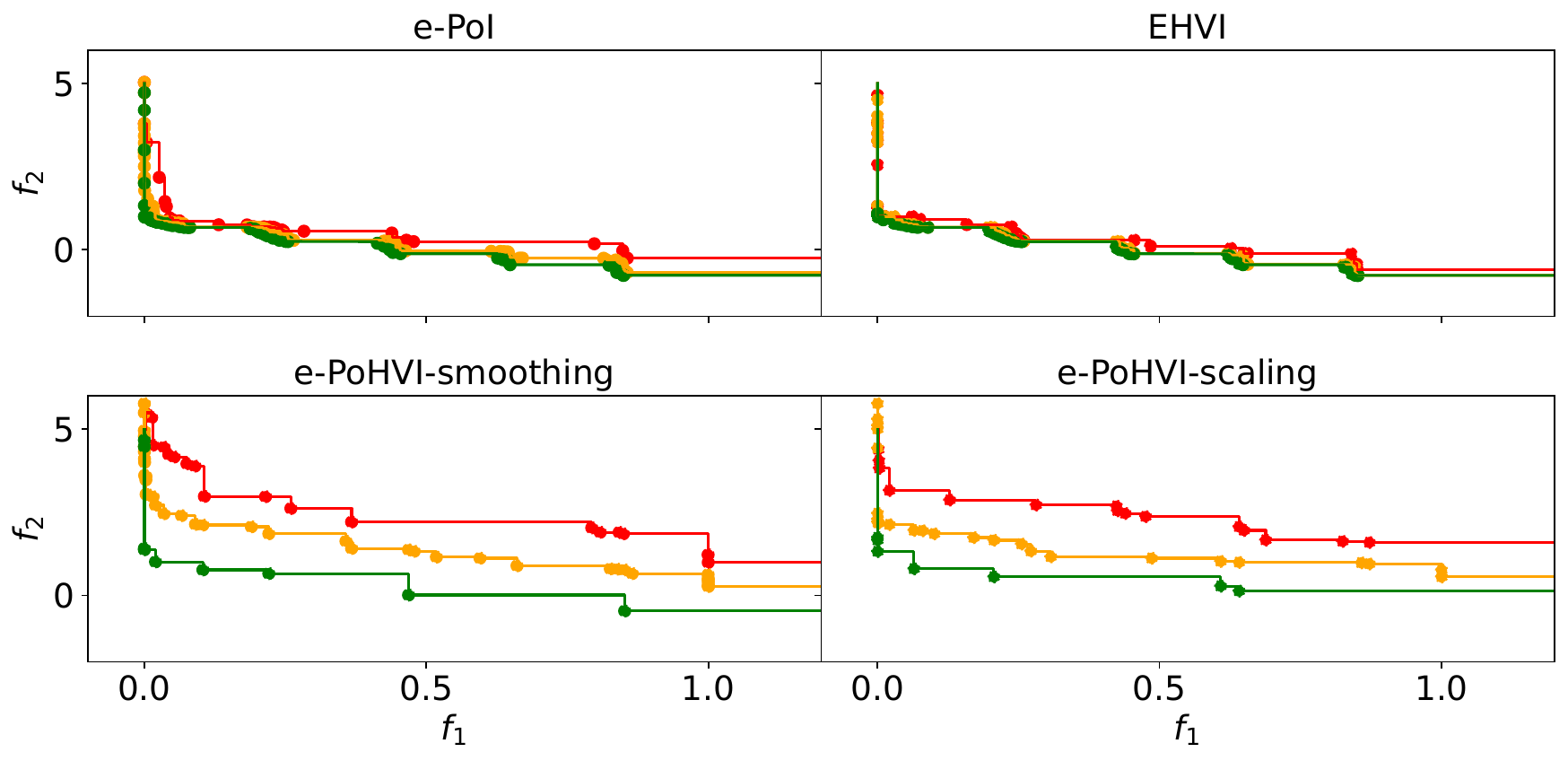}
         \caption{ZDT3}
     \end{subfigure}
      
    \caption{The best, median, and the worst empirical attainment curves on all the test problems, where \protect\greendot, \protect\orangedot, and \protect\reddot \enskip represent the best, the worst, and the median Pareto-front approximation sets over 15 independent runs, respectively.}
    \label{fig:eaf_plots}
\end{figure}

\begin{figure}
\ContinuedFloat
     \centering
     \begin{subfigure}[b]{0.8\textwidth}
         \centering
         \includegraphics[width=\textwidth]{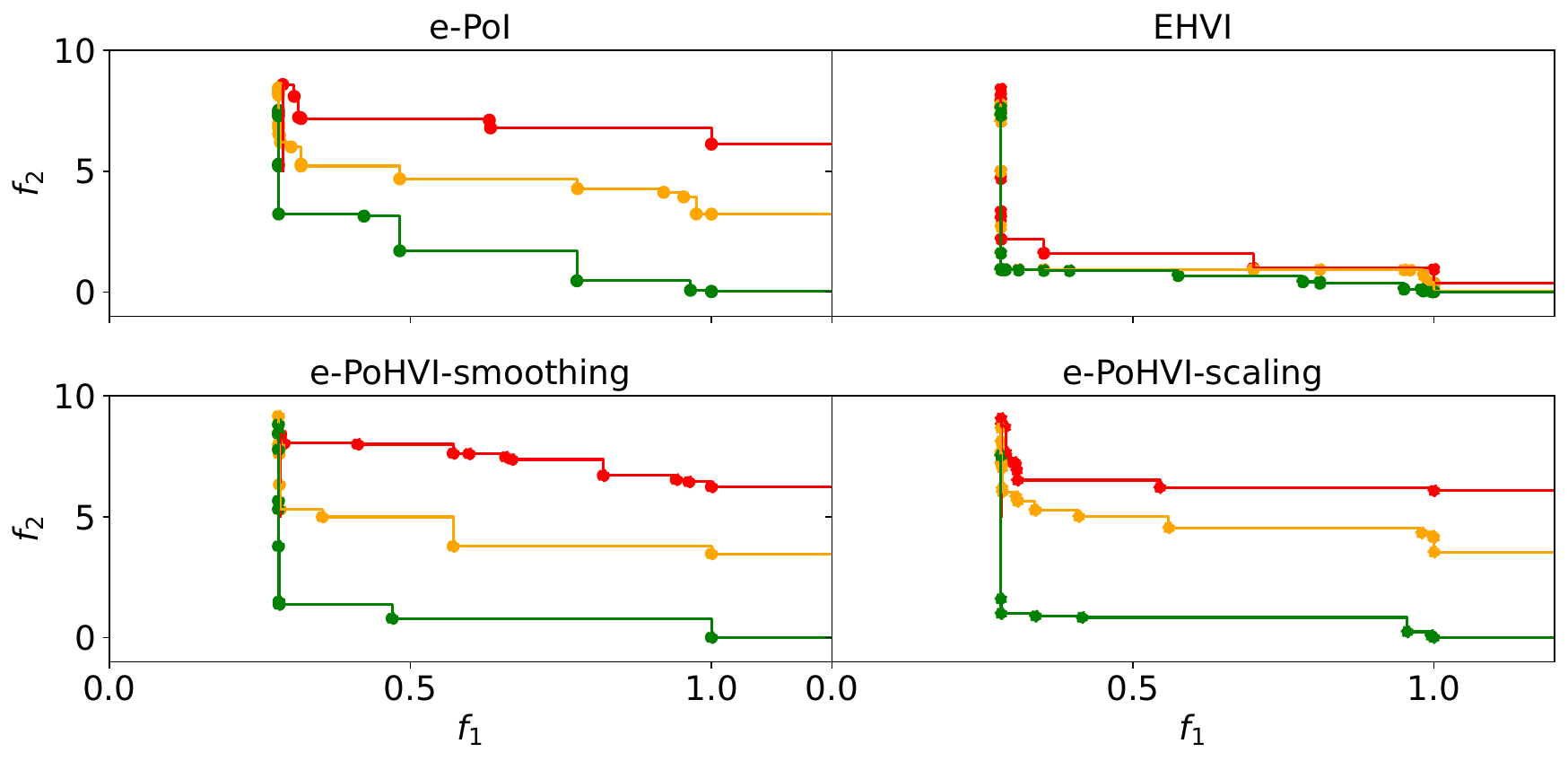}
         \caption{ZDT6}
     \end{subfigure}
     \hfill
      \begin{subfigure}[b]{0.8\textwidth}
         \centering
         \includegraphics[width=\textwidth]{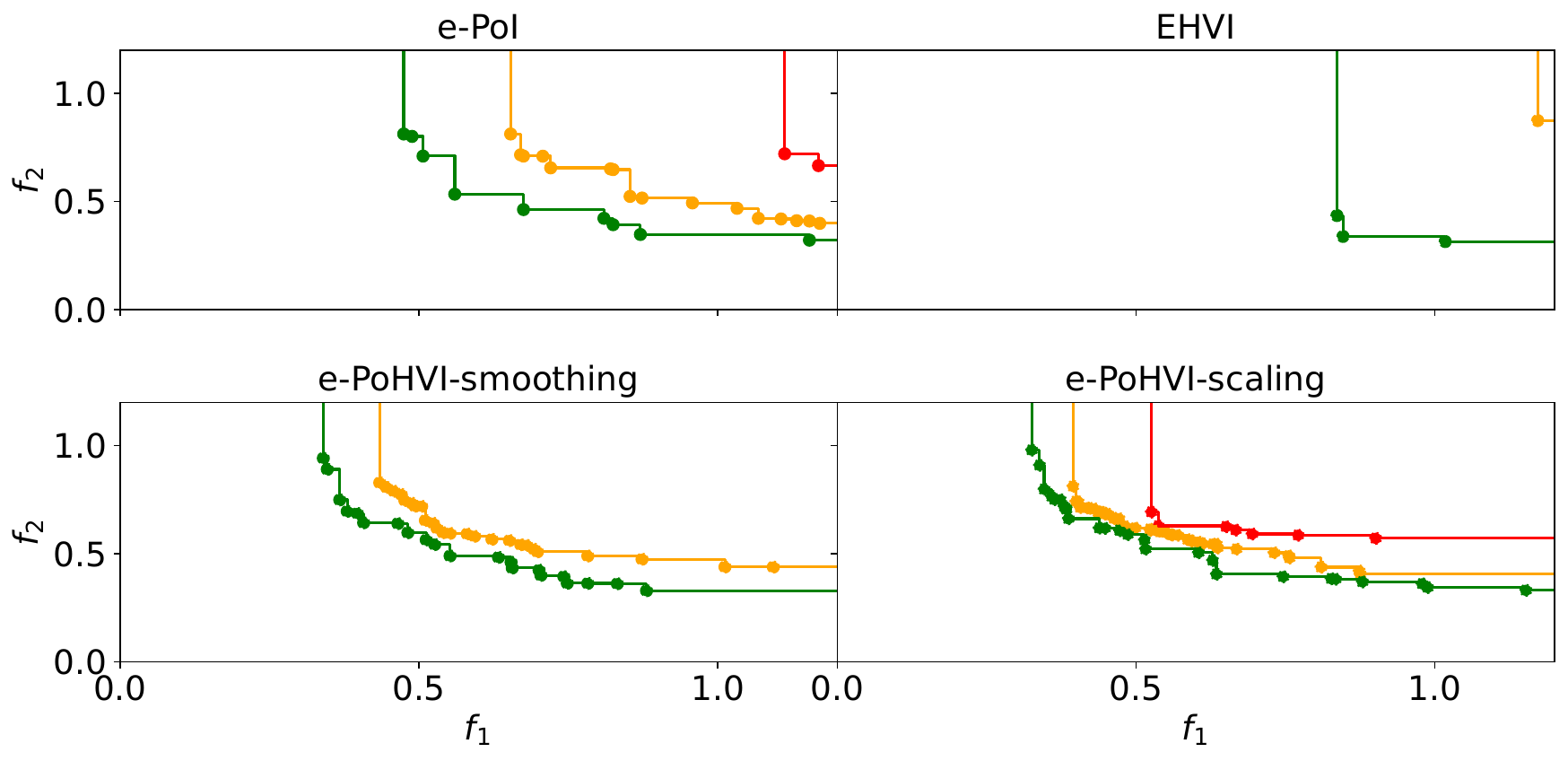}
         \caption{WOSGZ1}
     \end{subfigure}
     \hfill
     \begin{subfigure}[b]{0.8\textwidth}
         \centering
         \includegraphics[width=\textwidth]{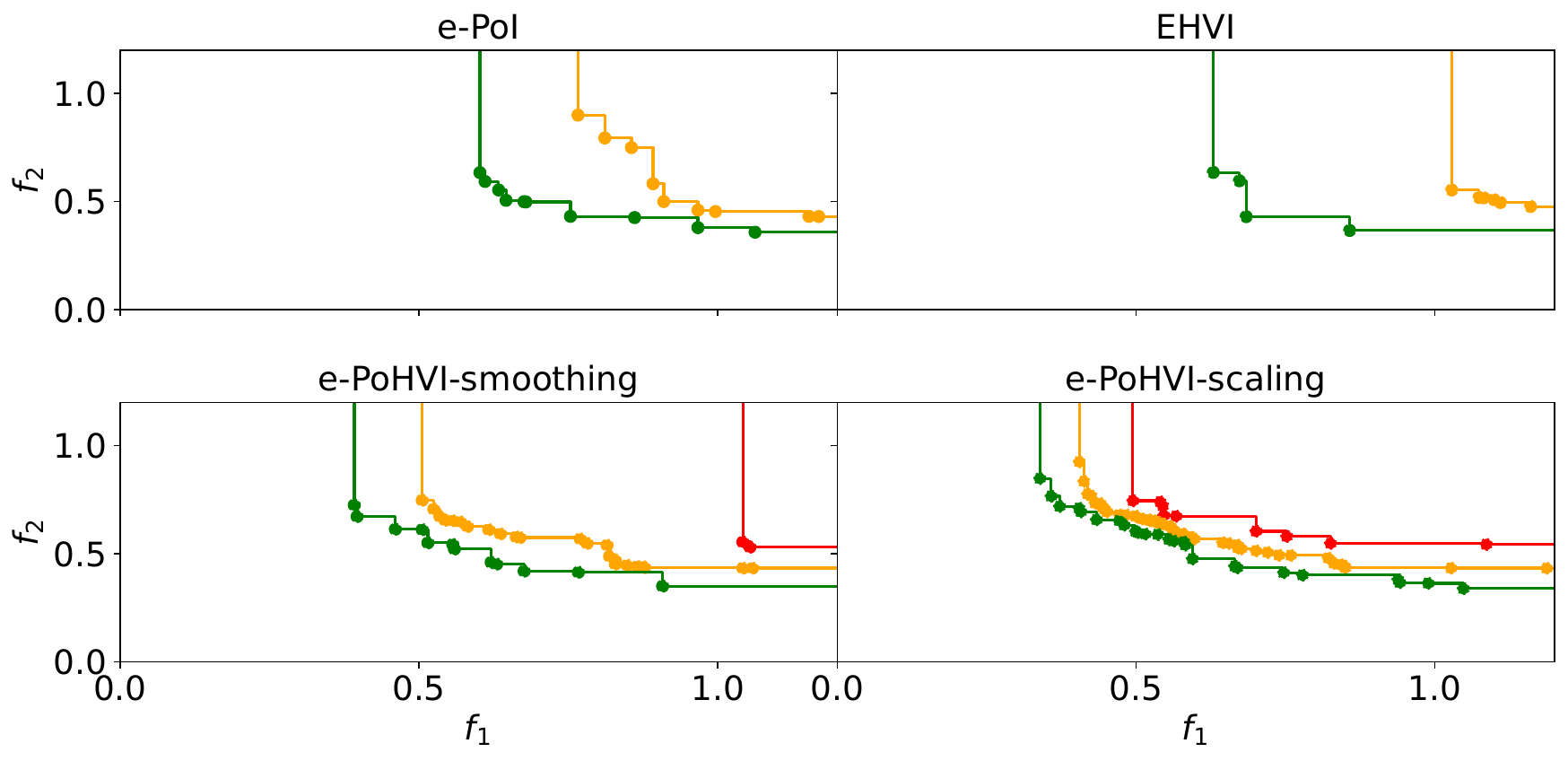}
         \caption{WOSGZ2}
     \end{subfigure}
        \caption{The best, median, and the worst empirical attainment curves on all the test problems, where \protect\greendot, \protect\orangedot, and \protect\reddot \enskip represent the best, the worst, and the median Pareto-front approximation sets over 15 independent runs, respectively.}
    \label{fig:eaf_plots}
\end{figure}

\begin{figure}
\ContinuedFloat
     \centering
  \begin{subfigure}[b]{0.8\textwidth}
         \centering
         \includegraphics[width=\textwidth]{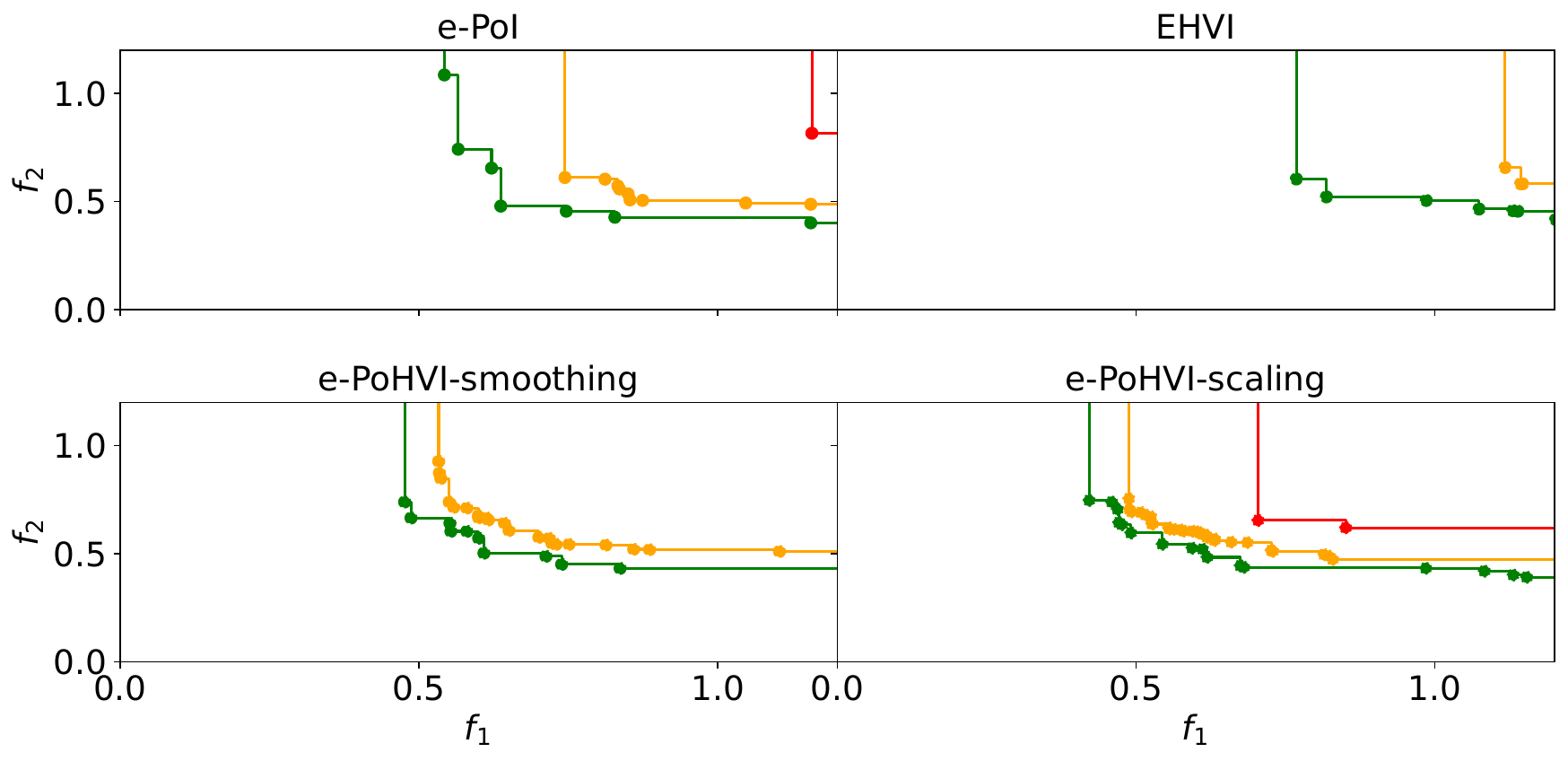}
         \caption{WOSGZ3}
    \end{subfigure}    
    \hfill
     \begin{subfigure}[b]{0.8\textwidth}
         \centering
         \includegraphics[width=\textwidth]{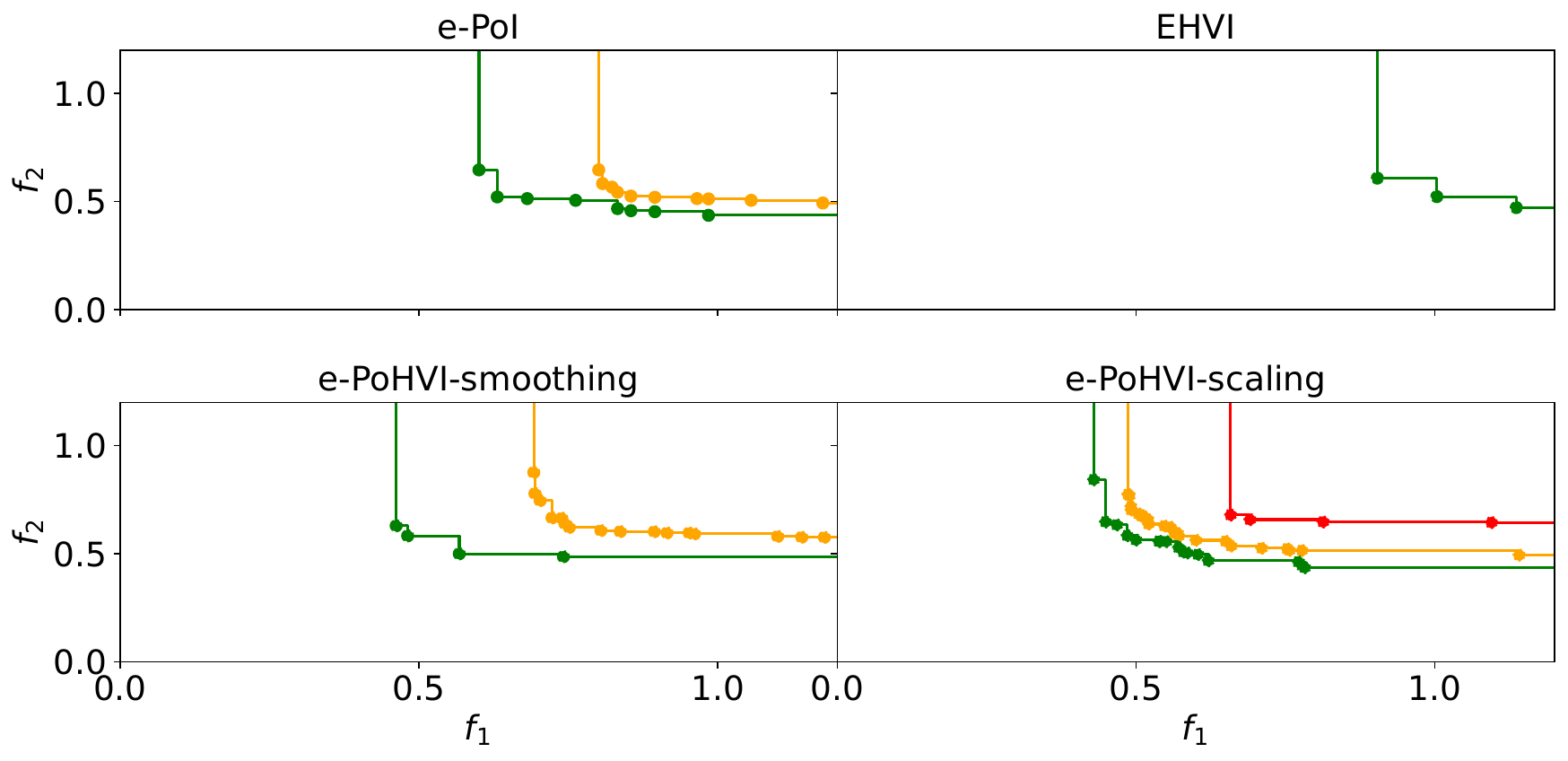}
         \caption{WOSGZ4}
     \end{subfigure}
     \hfill
     \begin{subfigure}[b]{0.8\textwidth}
         \centering
         \includegraphics[width=\textwidth]{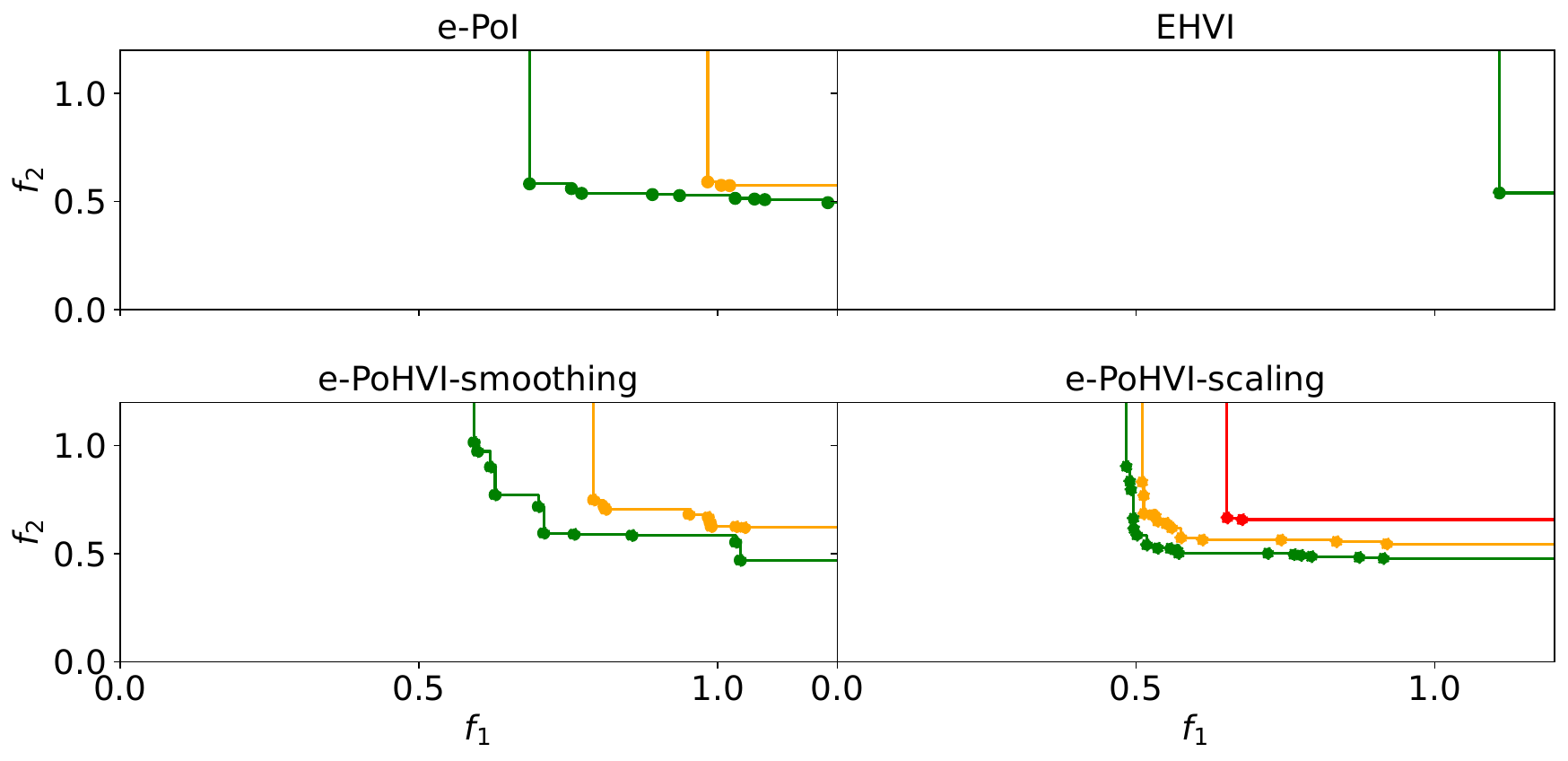}
         \caption{WOSGZ5}
     \end{subfigure}
        \caption{The best, median, and the worst empirical attainment curves on all the test problems, where \protect\greendot, \protect\orangedot, and \protect\reddot \enskip represent the best, the worst, and the median Pareto-front approximation sets over 15 independent runs, respectively.}
    \label{fig:eaf_plots}
\end{figure}

\begin{figure}
\ContinuedFloat
     \centering
     \begin{subfigure}[b]{0.8\textwidth}
         \centering
         \includegraphics[width=\textwidth]{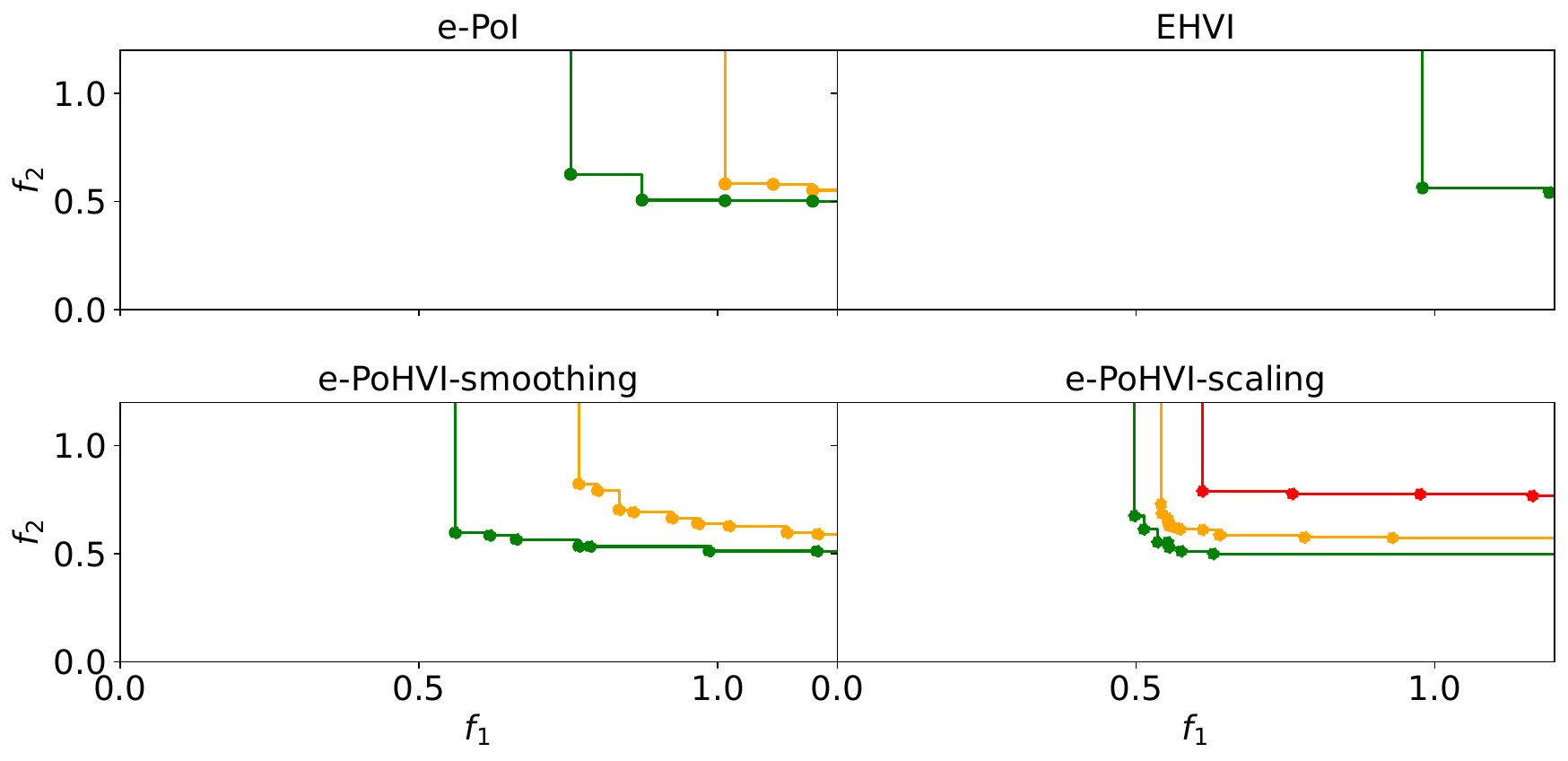}
         \caption{WOSGZ6}
     \end{subfigure}
     \hfill
     \begin{subfigure}[b]{0.8\textwidth}
         \centering
         \includegraphics[width=\textwidth]{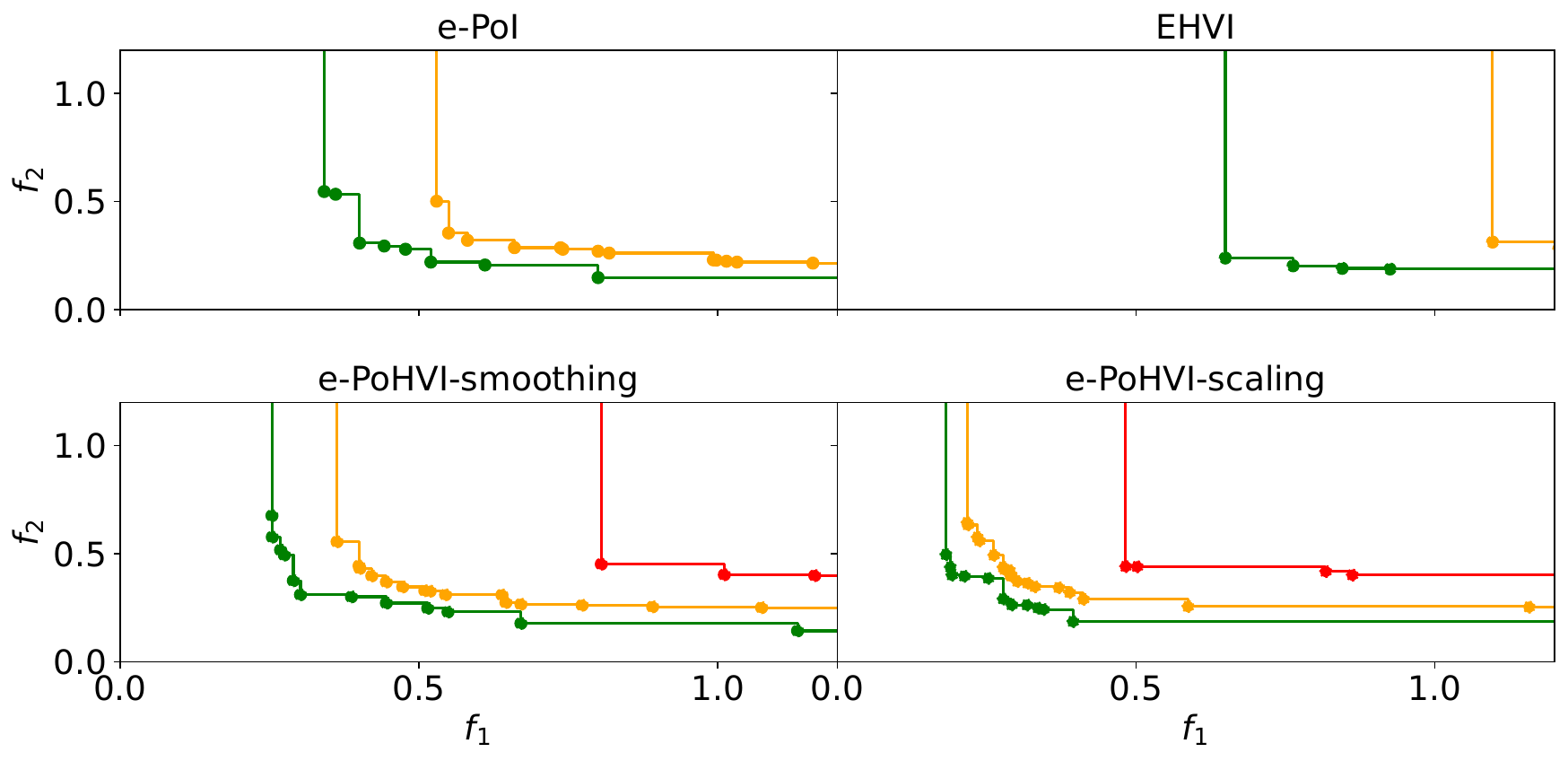}
         \caption{WOSGZ8}
     \end{subfigure}
     \hfill
     \begin{subfigure}[b]{0.8\textwidth}
         \centering
         \includegraphics[width=\textwidth]{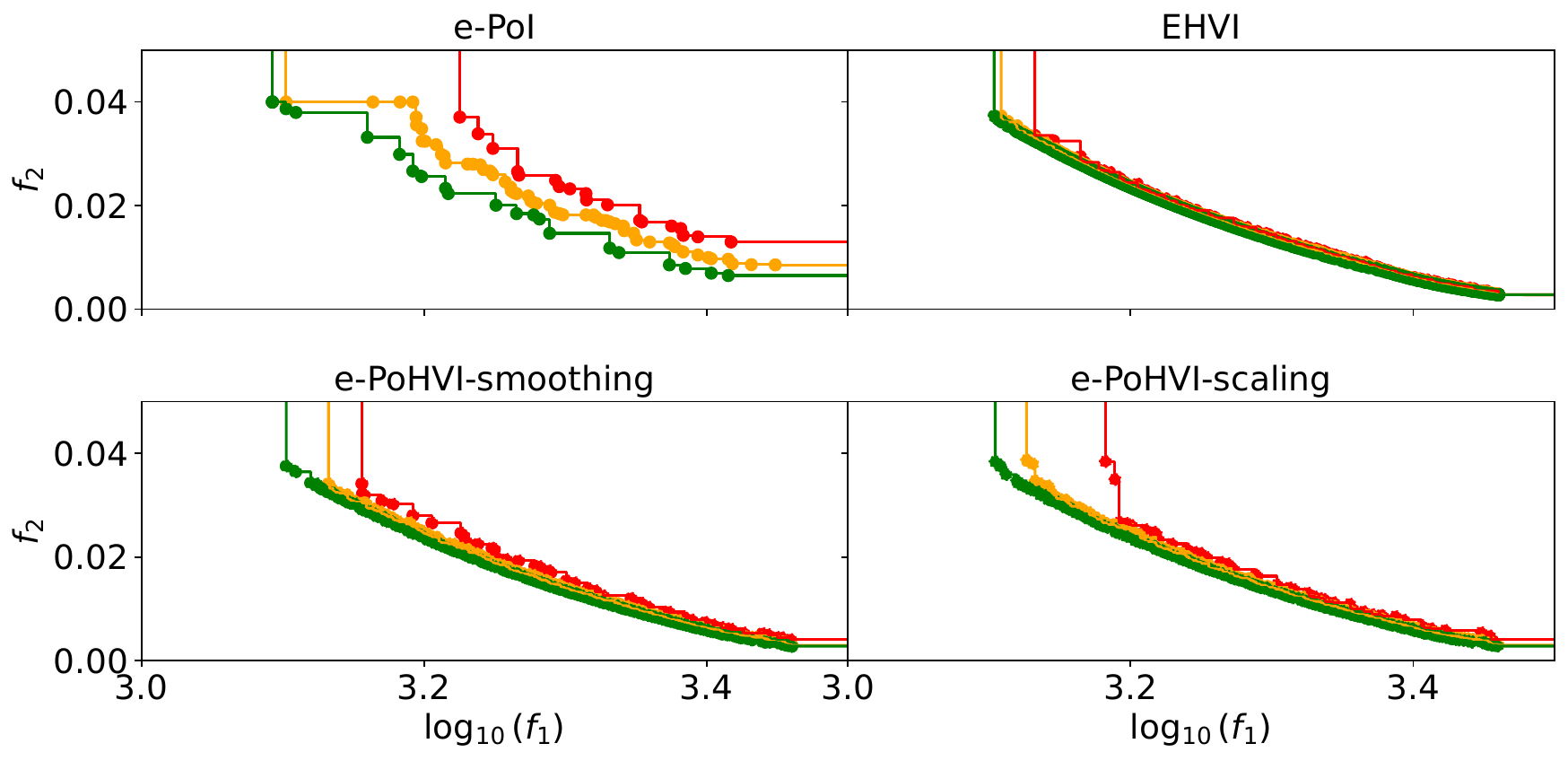}
         \caption{RE}
     \end{subfigure}
        \caption{The best, median, and the worst empirical attainment curves on all the test problems, where \protect\greendot, \protect\orangedot, and \protect\reddot \enskip represent the best, the worst, and the median Pareto-front approximation sets over 15 independent runs, respectively.}
    \label{fig:eaf_plots}
\end{figure}
\end{document}